\documentclass[twoside,11pt]{article}
\usepackage{journal2e_v1}

\usepackage{epsf}
\usepackage{epsfig}
\usepackage{amsmath}
\usepackage{subfigure}
\usepackage{amssymb}
\usepackage{multirow}
\usepackage{multicol}

\usepackage{algorithm}
\usepackage{algorithmic}

\def\b{{\bf b}}
\def\C{{\bf C}}

\def\I{{\bf I}}

\def\X{{\bf X}}

\def\x{{\bf x}}
\def\y{{\bf y}}

\def\u{{\bf u}}

\def\W{{\bf W}}

\def\0{{\bf 0}}
\def\1{{\bf 1}}

\def\AM{{\mathcal A}}

\def\ª{{\mathcal T}}

\def\XM{{\mathcal X}}

\def\RB{{\mathbb R}}

\def\bet{\mbox{\boldmath$\beta$\unboldmath}}
\def\epsi{\mbox{\boldmath$\epsilon$\unboldmath}}

\def\pii{\mbox{\boldmath$\pi$\unboldmath}}

\def\Si{\mbox{\boldmath$\Sigma$\unboldmath}}

\def\vps{\mbox{\boldmath$\varepsilon$\unboldmath}}

\def\argmin{\mathop{\rm argmin}}

\def\sgn{\mathrm{sgn}}

\def\GIG{\mathsf{GIG}}
\def\IG{\mathsf{IG}}
\def\EP{\mathsf{EP}}
\def\EIG{\mathsf{EGIG}}

\def\EG{\mathsf{EG}}
\def\GT{\mathsf{GT}}

\begin{document}
\title{EP-GIG Priors and Applications in Bayesian Sparse Learning}

\author{\name Zhihua Zhang,  Shusen Wang and Dehua Liu \\
\addr College of Computer Science \& Technology \\
Zhejiang University \\
Hangzhou, China 310027 \\
\texttt{\{zhzhang, wssatzju\}@gmail.com} \AND
\name Michael I. Jordan\\
\addr Computer Science Division and Department of Statistics \\
University of California, Berkeley \\
CA 94720-1776 USA \\
\texttt{jordan@stat.berkeley.edu}
}

\date{April 2012}

\maketitle

\begin{abstract}
In this paper
we propose a novel framework for the construction of  sparsity-inducing priors. In particular,
we define such priors as a mixture of exponential power distributions with a generalized inverse Gaussian density (EP-GIG).
EP-GIG is a  variant of generalized hyperbolic distributions, and
the special cases include Gaussian scale mixtures and Laplace scale mixtures.
Furthermore, Laplace scale mixtures can subserve a Bayesian framework for sparse learning
with nonconvex penalization.
The densities of EP-GIG can be explicitly expressed. Moreover, the corresponding posterior distribution also follows
a generalized inverse Gaussian distribution. These properties lead us to EM algorithms
for  Bayesian sparse learning. We show that these algorithms bear an interesting
resemblance to iteratively re-weighted $\ell_2$ or $\ell_1$ methods. In addition, we present two extensions for grouped 
variable selection and
logistic regression.
\end{abstract}

\begin{keywords} Sparsity priors, scale mixtures of exponential power distributions, generalized inverse Gaussian distributions, expectation-maximization algorithms, iteratively re-weighted minimization methods.
\end{keywords}

\section{Introduction}

In this paper we are concerned with sparse supervised learning problems
over a training dataset ${\XM}=\{(\x_i, y_i)\}_{i=1}^n$.
The point of departure for our work is the traditional formulation of
supervised learning as a regularized optimization problem:
\[
\min_{\b} \; \Big\{L(\b; \XM) +   P_{\lambda}(\b)\Big\},
\]
where $\b$ denotes the model parameter vector, $L(\cdot)$ the loss function penalizing data misfit, $P_{\lambda}(\cdot)$
the regularization term penalizing model complexity, and $\lambda>0$ the tuning parameter balancing the relative significance
of the loss function and the penalty.

Variable selection is a fundamental problem in high-dimensional learning
problems, and is closely tied to the notion that the data-generating
mechanism can be described using a sparse representation. In supervised
learning scenarios, the problem is to obtain sparse estimates for the
regression vector $\b$.  Given that it is NP-hard to use the $\ell_0$
penalty (i.e., the number of the nonzero elements of $\b$)~\citep{Weston:2003}, attention
has focused on use of the $\ell_1$ penalty~\citep{TibshiraniLASSO:1996}.
But in addition a number of studies have emphasized the advantages of
nonconvex penalties---such as the bridge penalty and the log-penalty---for
achieving sparsity~\citep{FuBridge:1998,Fan01,MazumderSparsenet:11}.

The regularized optimization problem can be cast into a maximum \emph{a posteriori}
(MAP) framework.  This is done by taking a Bayesian decision-theoretic approach
in which the loss function $L(\b; \XM)$ is based on the conditional
likelihood of the output $y_i$ and the penalty $P_{\lambda}(\b)$ is
associated with a prior distribution of $\b$. For example, the least-squares
loss function is associated with Gaussian likelihood, while there exists
duality between the $\ell_1$  penalty and the Laplace prior.

The MAP framework provides us with Bayesian underpinnings for the sparse
estimation problem.  This has led to Bayesian versions of the lasso, which
are based on expressing the Laplace prior as a scale-mixture of a Gaussian
distribution and an exponential density~\citep{AndrewsMallows:1974,West:1987}.  \cite{Figueiredo:2003}, and \cite{Kiiveri:2008} presented 
a Bayesian lasso based on the Expectation-Maximization (EM) algorithm.   \cite{CaronDoucet:icml} then considered an EM estimation 
with normal-gamma or
normal-inverse-gaussian priors.  In one latest work, 
\cite{PolsonScottTR}   proposed using generalized hyperbolic distributions,
variance-mean mixtures of Gaussians with generalized inverse Gaussian densities. Moreover,  \cite{PolsonScottTR} devised EM algorithms via data augmentation methodology. However, these treatments is  not fully Bayesian. \cite{LeeCaronDoucetHolmes:2010} referred to such  
a method based on MAP estimation as a quasi-Bayesian approach. Additionally, an empirical-Bayes sparse learning
was  developed by \cite{TippingJMLR:2001}.

Recently,
\cite{Park:2008} and \cite{Hans:2009} proposed  full Bayesian lasso models based on Gibbs sampling.
Further work by 
\cite{GriffinBrownBA:2010} involved
the use of a family of normal-gamma priors as a generalization of the Bayesian lasso. This prior has been also used
by \cite{ArchambeuNIPS:2009} to develop sparse probabilistic projections.
In the work of \cite{CarvalhoPolson:2010}, the authors proposed  horseshoe priors which are
a mixture of normal distributions and a half-Cauchy density on the positive reals with scale parameter.
\cite{KyungBA:2010} conducted performance
analysis of Bayesian lassos in depth theoretically and empirically.

There has also been work on nonconvex penalties within a Bayesian framework.
\cite{ZouLi:2008} derived their local linear approximation
(LLA) algorithm by combining the EM algorithm with an inverse  Laplace
transformation. In particular, they showed that the bridge penalty can be
obtained by mixing the Laplace distribution with a stable distribution.
Other authors have shown that the prior induced from the log-penalty
has an interpretation as a scale mixture of Laplace distributions with an
inverse gamma density~\citep{CevherNIPS:2009,GarriguesfNIPS:2010,LeeCaronDoucetHolmes:2010,ArmaganDunsonLee}.
Additionaly, \cite{GriffinBrownTA:2010} devised a family
of normal-exponential-gamma priors for a Bayesian adaptive lasso~\citep{ZouJasa:2006}.
\cite{PolsonBS:2010,PolsonScott:2011} provided a unifying framework for the construction of sparsity priors using L\'{e}vy processes.

In this paper we develop a novel framework for constructing sparsity-inducing priors. Generalized
inverse Gaussian (GIG) distributions~\citep{Jorgensen:1982} are conjugate with respect to an exponential power (EP)
distribution~\citep{BoxTiao:1992}---an extension of Gaussian and Laplace distributions. Accordingly,
we propose a family of distributions that we refer to as \emph{EP-GIG}.
In particular, we define EP-GIG  as a scale mixture of EP distributions with a GIG density, and derive their explicit densities.
EP-GIG can be regarded as a variant of generalized hyperbolic distributions, and
include Gaussian scale mixtures and  Laplacian scale mixtures as special cases. The Gaussian scale mixture
is a class of generalized hyperbolic distributions~\citep{PolsonScottTR}
and its special cases include  normal-gamma distributions~\citep{GriffinBrownBA:2010} as well as the Laplacian distribution.
The  generalized double Pareto distribution in \citep{CevherNIPS:2009,ArmaganDunsonLee,LeeCaronDoucetHolmes:2010} and the 
bridge distribution inducing the $\ell_{1/2}$ bridge penalty~\citep{ZouLi:2008} are special cases of Laplacian scale mixtures.  In Appendix~\ref{app:cases},
we devise a set of now EP-GIG priors.

Since GIG priors are conjugate with respect to EP distributions, it is feasible to
apply EP-GIG to Bayesian sparse learning. Although it has been illustrated that fully Bayesian sparse learning methods based on Markov chain Monte Carlo  sampling
work well,
our main  attention is paid to a quasi-Bayesian approach. The important purpose is to explore the equivalent 
relationship between the MAP estimator and the classical regularized
estimator  in depth.  In particular,  using the property of  EP-GIG as a
scale-mixture of exponential power distributions, we devise EM algorithms for finding a sparse MAP estimate of
$\b$.

When we set the exponential power distribution to be the Gaussian distribution, the resulting
EM algorithm is closely related to the iteratively  re-weighted $\ell_2$ minimization methods
in~\citep{Daubechies:2010,ChartrandICASSP:2008,WipfNagarajan:2010}. When we employ the Laplace distribution as a special exponential power distribution, we  obtain an EM algorithm which is identical to the iteratively re-weighted $\ell_1$ minimization method
in~\citep{CandesWakinBoyd:2008}. Interestingly, using a bridge distribution of order $q$ ($0<q<1$)
results in an EM algorithm, which in turn corresponds to an iteratively re-weighted $\ell_q$ method.

We also  develop  hierarchical Bayesian approaches for grouped variable selection~\citep{YuanLin:2006}
and penalized logistic regression by using EP-GIG priors. We apply  our proposed EP-GIG priors in Appendix~\ref{app:cases} 
to conduct experimental analysis.
The experimental results  validate that
those proposed EP-GIG priors which induce  nonconvex penalties  are potentially feasible and effective in sparsity modeling.
Finally, we would like to highlight  that our work offers several
important theorems as follows.

\begin{enumerate}

\item Theorems~\ref{thm:02}, \ref{thm:5} and \ref{thm:6}  recover the relationship of EP-GIG with the corresponding EP at the limiting case. 
These theorems  extend the fact that
the $t$-distribution degenerates to Gaussian distribution as the degree of freedom approaches infinity.
\item Theorem~\ref{thm:00} proves that an exponential power distribution of order $q/2$ ($q>0$)  can be always represented a scale mixture of
exponential power distributions of order $q$ with a gamma mixing density.
\item The first part of Theorem~\ref{thm:2} shows that GIG is conjugate with respect to EP, while the second part then offers a theoretical evidence  of relating EM algorithms
with iteratively re-weighted minimization methods under our  framework.
\item Theorem~\ref{thm:monoto}  shows that the negative log EP-GIG can induce a class of sparsity penalties.  Especially,  it shows a class of nonconvex penalties. Finally, Theorem~\ref{thm:oracle1} establishes the oracle properties of the sparse estimator based on Laplace scale mixture priors.  
\end{enumerate}

The paper is organized as follows.
A brief reviewe about exponential power distributions and generalized inverse Gaussian distributions is given in Section~\ref{sec:problem}.
Section~\ref{sec:epgig}
presents EP-GIG distributions and their some properties,  Section~\ref{sec:em} develops our EM algorithm for Bayesian sparse learning,  
and Section~\ref{sec:related}
discusses the equivalent relationship between the EM  and iteratively re-weighted minimization methods.  In Section~\ref{sec:experiment} 
we conduct our
experimental evaluationsd.   Finally, we conclude our work in
Section~\ref{sec:concl},  defer all proofs to Appendix~\ref{app:proof}, and provide several new
sparsity priors in Appendix~\ref{app:cases}.

\section{Preliminaries}  \label{sec:problem}

Before  presenting  EP-GIG priors for sparse modeling of regression vector $\b$,
we give some notions such as  the exponential power  (EP)  and generalized inverse Gaussian (GIG) distributions.

\subsection{Exponential Power Distributions} \label{sec:ep}

For a univariate  random variable $b \in \RB$, it is said to  follow an EP distribution
if the density is specified by
\[
p(b) = \frac{ \eta^{-1/q} }{2^{\frac{q+1}{q}} \Gamma(\frac{q{+}1} {q})} \exp(- \frac{1}{2\eta} |b
-u|^q)=\frac{q}{2} \frac{(2\eta)^{-\frac{1}{q}}}{ \Gamma(\frac{1}{q})} \exp(- \frac{1}{2\eta} |b
-u|^q),
\]
with $\eta>0$. In the literature~\citep{BoxTiao:1992},  it is typically assumed that $q \geq 1$. However, we find that it is able to relax this assumption
into $q>0$ without any obstacle. Thus, we assume $q>0$  for our purpose.    Moreover, we will
set $u=0$.

The distribution is denoted by ${\EP}(b|u, \eta, q)$.
There are  two classical special cases: the Gaussian distribution arises when
$q=2$ (denoted $N(b|u, \eta)$) and the Laplace distribution arises
when $q=1$  (denoted $L(b|u, \eta)$).  As for the case that $q<1$, the corresponding density induces a bridge penalty for $b$.
We thus refer to it as the bridge distribution.

\subsection{Generalized Inverse Gaussian Distributions} \label{sec:gig}

We first let $G(\eta|\tau, \theta)$ denote the gamma distribution whose density is
\[
p(\eta) = \frac{\theta^{\tau}}{\Gamma(\tau)} \eta^{\tau-1} \exp(-\theta \eta), \quad \tau, \theta>0,
\]
and  $\IG(\eta|\tau, \theta)$  denote  the inverse gamma distribution
whose density is
\[
p(\eta) = \frac{\theta^{\tau}}{\Gamma(\tau)} \eta^{-(1+\tau)} \exp(-\theta \eta^{-1}), \quad \tau, \theta>0.
\]
An interesting property of the gamma and inverse gamma distributions is given as follows.
\begin{proposition} \label{pro:7} Let $\lambda>0$. Then
\begin{enumerate}
\item[\emph{(1)}] $\lim_{\tau \rightarrow \infty} G(\eta|\tau, \tau \lambda) = \delta(\eta|1/\lambda)$.
\item[\emph{(2)}]  $\lim_{\tau \rightarrow \infty} \IG(\eta|\tau, \tau /\lambda) = \delta(\eta|1/\lambda)$.
\end{enumerate}
Here $\delta(\eta|a) $ is the Dirac delta function; namely,
\[
\delta(\eta|a) = \left\{\begin{array} {ll} \infty & \mbox{if }  \eta =a, \\ 0 & \mbox{otherwise}. \end{array}  \right.\]
\end{proposition}

We now consider the GIG distribution.  The  density of GIG  is defined   as
\begin{equation} \label{eqn:gig}
p(\eta) = \frac{(\alpha/\beta)^{\gamma/2}}{2 K_{\gamma}(\sqrt{\alpha \beta})} \eta^{\gamma-1} \exp(-(\alpha \eta + \beta \eta^{-1})/2), 
\; \eta>0,
\end{equation}
where $K_{\gamma}(\cdot)$ represents the modified Bessel function of the second kind with the index $\gamma$.
We denote this distribution by ${\GIG}(\eta|\gamma, \beta, \alpha)$.
It is well known that its special cases
include the gamma distribution $G(\eta|\gamma, \alpha/2)$ when $\beta=0$ and $\gamma>0$,
the inverse gamma distribution $\IG(\eta|-\gamma, \beta/2)$  when $\alpha=0$ and $\gamma <0$,  the inverse Gaussian
distribution when $\gamma=-1/2$, and the hyperbolic distribution when $\gamma=0$. Please refer to~\cite{Jorgensen:1982}
for the details.

Note in particular that the pdf of the inverse Gaussian ${\GIG}(\eta|{-}1/2, \beta, \alpha)$ is
\begin{equation} \label{eqn:ig}
p(\eta) = \Big(\frac{\beta}{2 \pi}\Big)^{1/2} \exp(\sqrt{\alpha \beta}) \eta^{-\frac{3}{2}}
\exp(-(\alpha \eta + \beta \eta^{-1})/2), \; \beta>0,
\end{equation}
and the pdf of ${\GIG}(\eta|1/2, \beta, \alpha)$ is
\begin{equation} \label{eqn:iig}
p(\eta) = \Big(\frac{\alpha}{2 \pi}\Big)^{1/2} \exp(\sqrt{\alpha \beta}) \eta^{-\frac{1}{2}}
\exp(-(\alpha \eta + \beta \eta^{-1})/2), \; \alpha>0.
\end{equation}
Note moreover that ${\GIG}(\eta|{-}1/2, \beta, 0)$ and ${\GIG}(\eta|1/2, 0, \alpha)$
degenerate to $\IG(\eta|1/2, \beta/2)$ and $G(\eta|1/2, \alpha/2)$, respectively.

As an extension of Proposition~\ref{pro:7}, we have the limiting property of GIG as
follows.
\begin{proposition} \label{pro:08}
Let $\alpha>0$ and $\beta>0$. Then
\begin{enumerate}
\item[\emph{(1)}] $\lim_{\gamma \rightarrow +\infty} \GIG(\eta|\gamma, \beta, \gamma \alpha) = \delta(\eta|2/\alpha)$.
\item[\emph{(2)}]  $\lim_{\gamma \rightarrow -\infty} \GIG(\eta|\gamma, -\gamma \beta, \alpha) = \delta(\eta|\beta/2)$.
\end{enumerate}
\end{proposition}

We now present an alternative expression for the GIG density that  is also interesting. Let $\psi=\sqrt{\alpha \beta}$ and $\phi=\sqrt{\alpha/\beta}$.
We can rewrite  the density of  $\GIG(\eta|\gamma, \beta, \alpha)$ as
\begin{equation} \label{eqn:gig2}
p(\eta) = \frac{\phi^{\gamma}}{2 K_{\gamma}(\psi)} \eta^{\gamma-1} \exp(- \psi(\phi \eta +  (\phi \eta)^{-1})/2), \; \eta>0.
\end{equation}

\begin{proposition} \label{pro:09}
Let $p(\eta)$ be defined by (\ref{eqn:gig2}) where $\phi$ is a positive constant and $\gamma$ is a any constant. Then
\[
\lim_{\psi +\infty} p(\eta) = \delta(\eta|\phi).
\]   
\end{proposition}

Finally, let us consider that the case $\gamma=0$.  Furthermore,  letting $\psi \rightarrow 0$, we can see that $p(\eta) \propto 1/\eta$, an improper prior.
Note that this improper prior can regarded as the Jeffreys prior because the Fisher information
of $\EP(b|0, \eta)$ with respect to $\eta$ is $\eta^{-2}/q$.

\section{EP-GIG Distributions} \label{sec:epgig}

We now develop a family of distributions by mixing the exponential power ${\EP}(b|0, \eta, q)$ with the generalized inverse Gaussian ${\GIG}(\eta|\gamma, \beta, \alpha)$.
The marginal density of $b$ is currently defined by
\[
p(b) = \int_{0}^{+\infty}{{\EP}(b|0, \eta, q) {\GIG}(\eta|\gamma, \beta, \alpha) d \eta}.
\]
We refer to this distribution as the EP-GIG and denote it by ${\EIG}(b|\alpha, \beta, {\gamma}, q)$.
\begin{theorem} \label{thm:1}
Let $b \sim {\EIG}(b|\alpha, \beta, {\gamma}, q)$. Then its density is
\begin{equation} \label{eqn:epgig}
p(b) = \frac{K_{\frac{\gamma q{-}1}{q}}(\sqrt{ \alpha(\beta {+}|b|^q)})} {2^{\frac{q+1}{q}} \Gamma(\frac{q{+}1} {q}) 
K_{\gamma}(\sqrt{\alpha \beta})} \frac{\alpha^{1/(2q)}}{\beta^{{\gamma}/2} } [\beta {+}|b|^q]^{(\gamma q-1)/(2q)}.
\end{equation}
\end{theorem}

The following theorem establishes an important relationship of an EP-GIG with the corresponding  EP distribution. It
is an extension of the relationship of a $t$-distribution with the Gaussian distribution. That is,

\begin{theorem} \label{thm:02} Consider EP-GIG distributions. Then
\begin{enumerate}
\item[\emph{(1)}] $\lim_{\gamma \rightarrow +\infty}  \EIG(b|\gamma \alpha, \beta, \gamma, q) = \EP(b|0, 2/\alpha, q)$;
\item[\emph{(2)}]  $\lim_{\gamma \rightarrow -\infty} \EIG(b|\alpha, -\gamma \beta, \gamma, q) = \EP(b|0, \beta/2, q)$.
\item[\emph{(3)}]  $\lim_{\psi \rightarrow + \infty} \EIG(b|\alpha,  \beta, \gamma, q) = \EP(b|0, \phi, q)$ where $\psi=\sqrt{\alpha \beta}$
and $\phi=\sqrt{\alpha/\beta} \in (0, \infty)$.
\end{enumerate}
\end{theorem}

EP-GIG can be regarded as a variant of generalized hyperbolic distributions~\citep{Jorgensen:1982}, because when $q=2$  
EP-GIG is  generalized hyperbolic distributions---a class of Gaussian scale mixtures. However,  EP-GIG becomes
a class of Laplace scale mixtures when $q=1$. 

In Appendix~\ref{app:cases} we  present several  new concrete EP-GIG distributions, obtained
from particular settings of ${\gamma}$ and $q$.
We now consider the two special cases that mixing density is either a gamma distribution or an inverse gamma distribution.
Accordingly, we have two special EP-GIG:  exponential  power-gamma distributions and exponential power-inverse gamma distributions.

\subsection{Generalized $t$ Distributions}
\label{sec:epig}

We first consider an important family of EP-GIG distributions, which are  scale mixtures
of exponential power $\EP(b|u, \eta, q)$ with  inverse gamma $\IG(\eta|\tau/2, \tau/(2\lambda))$.
Following the terminology of \cite{LeeCaronDoucetHolmes:2010}, we
refer them as \emph{generalized $t$ distributions}
 and denote them by $\GT(b|u, \tau/\lambda, \tau/2, q)$. Specifically, the density of the generalized $t$ is

\begin{equation} \label{eqn:ep-ig}
p(b) = \int{\EP(b|u, \eta, q) \IG(\eta|\tau/2, \tau/(2\lambda)) d \eta} = \frac{q}{2} \frac{\Gamma(\frac{\tau}{2} {+}\frac{1}{q})} {\Gamma(\frac{\tau}{2}) \Gamma(\frac{1}{q})}  \Big(\frac{\lambda}{\tau}\Big)^{\frac{1}{q}}
\Big(1 + \frac{\lambda}{\tau} |b-u|^q \Big)^{-(\frac{\tau}{2}+\frac{1}{q})}
\end{equation}
where $\tau>0$, $\lambda>0$ and $q>0$.
Clearly,  when $q=2$ the generalized $t$ distribution becomes to a $t$-distribution. Moreover, when $\tau=1$, it is the Cauchy distribution.

On the other hand, when $q=1$,  \cite{CevherNIPS:2009} and \cite{ArmaganDunsonLee} called the resulting
generalized $t$ \emph{generalized double Pareto distributions} (GDP).
The density of  GDP  is specified as
\begin{equation} \label{eqn:hal}
p(b) = \int_{0}^{\infty}{L(b|0, \eta)    \IG(\eta|\tau/2, \tau/(2\lambda))} d \eta  = \frac{\lambda}{4}\Big (1+ \frac{\lambda |b|}{\tau} \Big)^{-(\tau/2+1)}, \quad \lambda>0,
\tau>0.
\end{equation}
Furthermore,  we  let $\tau=1$; namely, $\eta \sim \IG(\eta|1/2, 1/(2\lambda))$.    As a result, we obtain
\[
p(b)=\frac{\lambda} {4} (1+\lambda |b|)^{-3/2},
\]

It is well known that the limit of the $t$-distribution at $\tau \rightarrow \infty$
is the normal distribution. We find that we are able to extend this property to the generalized $t$ distribution. In particular, we
have the following theorem, which is in fact a corollary of the first part of Theorem~\ref{thm:02}.

\begin{theorem} \label{thm:5} Let  the generalized $t$ distribution  be defined in (\ref{eqn:ep-ig}). Then, for $\lambda>0$ and $q>0$,
\[
\lim_{\tau \rightarrow \infty} \GT(b|u, \tau/\lambda, \tau/2, q) = \EP(b|u, 1/\lambda, q).
\]
\end{theorem}

Thus, as a special case of Theorem~\ref{thm:5} in $q=1$, we have
\[
\lim_{\tau \rightarrow \infty} \GT(b|u, \tau/\lambda, \tau/2, 1) = L(b|u, 1/\lambda).
\]

\subsection{Exponential Power-Gamma  Distributions} \label{sec:ng}

In particular, the density of  the exponential power-gamma  distribution
is defined by
\[
p(b|\gamma, \alpha) = \int_{0}^{\infty}{\EP(b|0, \eta, q)    G(\eta|\gamma, \alpha/2)} d \eta
=   \frac{\alpha^{\frac{q \gamma  + 1}{2 q}}
|b|^{\frac{q \gamma-1}{2} } }  {2^{\frac{q\gamma {+}1}{q} } \Gamma(\frac{q{+}1}{q})  \Gamma(\gamma)}
K_{\gamma-\frac{1}{q}} (\sqrt{\alpha |b|^q }),   \gamma, \alpha>0.
\]
We denote the distribution by $\EG(b|\alpha, \gamma, q)$.
As a result,  the density of  the normal-gamma distribution~\citep{GriffinBrownBA:2010} is
\begin{equation} \label{eqn:ng}
p(b|\gamma, \alpha) = \int_{0}^{\infty}{N(b|0, \eta)    G(\eta|\gamma, \alpha/2)} d \eta
=   \frac{\alpha^{\frac{2 \gamma  + 1}{4}}   |b|^{\gamma-\frac{1}{2}}  } {2^{\gamma-\frac{1}{2}} \sqrt{\pi}
\Gamma(\gamma)}  K_{\gamma-\frac{1}{2}} (\sqrt{\alpha} |b|),  \quad \gamma, \alpha>0.
\end{equation}
As the application of the second part of Theorem~\ref{thm:02} in this case,  we can obtain the following theorem.  
\begin{theorem} \label{thm:6}
Let $\EG(b|\lambda \gamma, \gamma/2, q) =  \int_{0}^{\infty}{\EP(b|0, \eta, q)    G(\eta|\gamma/2, \lambda \gamma/2)} d \eta$
with $\lambda>0$. Then
\[
\lim_{\gamma \rightarrow \infty } \EG(b|\lambda \gamma, \gamma/2, q) = \EP(b|0, 1/\lambda, q).
\]
\end{theorem}

It is easily seen that  when we let  $\gamma=1$,  the normal-gamma distribution degenerates to the Laplace distribution
$L(b|0,  \alpha^{-1/2}/2)$.
In addition, when $q=1$ and $\gamma=3/2$ which implies that $[b|\eta] \sim L(b|0, \eta)$ and $\eta \sim G(\eta|3/2, \alpha/2)$,
we  have
\begin{equation} \label{eqn:eig13}
p(b|\alpha)= \frac{\alpha}{4} \exp(-\sqrt{\alpha |b|})= \int_{0}^{+\infty} {L(b|0, \eta) G(\eta|3/2, \alpha/2) d \eta}.
\end{equation}
Obviously, the current exponential power-gamma  is identical to exponential power distribution $\EP(b|0,  \alpha^{-1/2}/2, 1/2 )$,  
a bridge distribution with $q=1/2$.
Interestingly, we can extend this relationship between the Gaussian and Laplace as we as between the Laplace and $1/2$-bridge
to the general case. That is,
\begin{theorem} \label{thm:00} Let  $\gamma =\frac{1}{2} +\frac{1}{q}$. Then,
\[
\EP(b|0, \alpha^{-1/2}/2, q/2) = \frac{q \alpha^{{1}/{q}}} {4 \Gamma({2}/{q})} \exp(-\sqrt{\alpha |b|^q})= \int_{0}^{+\infty} 
{\EP(b|0, \eta, q) G(\eta|\gamma, \alpha/2) d \eta}.\]
\end{theorem}
This theorem implies that a $q/2-$bridge distribution can be represented as a scale mixture of  $q-$bridge distributions.
 A class of  important settings are  $q=2^{1{-}m}$ and $\gamma =\frac{1}{2} +\frac{1}{q}= \frac{1+ 2^m}{2}$ where
$m$  is any nonnegative integer.

\subsection{Conditional Priors, Marginal Priors and Posteriors} \label{sec:poster}

We now study the posterior distribution of $\eta$ conditioning on $b$. It is immediate that
the  posterior distribution follows ${\GIG}(\eta|(\gamma q{-}1)/q, (\beta+|b|^q), \alpha)$. This implies that GIG distributions 
are conjugate with respect to the EP distribution.
We note that in the cases $\gamma=1/2$ and $q=1$ as well as $\gamma=0$ and $q = 2$,
the posterior distribution is ${\GIG}(\eta|{-}1/2, (\beta+|b|^q), \alpha)$.
In the cases $\gamma=3/2$ and $q=1$ as well as $\gamma=1$ and $q = 2$,
the posterior distribution is ${\GIG}(\eta|1/2, (\beta+|b|^q), \alpha)$. 
When $\gamma=-1/2$ and $q=1$ or $\gamma=-1$ and $q=2$, the  posterior distribution is  ${\GIG}(\eta|{-}3/2, (\beta+|b|^q), \alpha)$.

Additionally, we have the following theorem.
\begin{theorem} \label{thm:2}
Suppose that $b|\eta \sim \EP(b|0, \eta, q)$ and $\eta \sim \GIG(\eta|\gamma, \beta, \alpha)$. Then
\begin{enumerate}
\item[\emph{(i)}]   $b \sim \EIG(b|\alpha, \beta, \gamma, q)$ and $\eta|b \sim \GIG(\eta|(\gamma q{-}1)/q, (\beta+|b|^q), \alpha)$.
\item[\emph{(ii)}]  $\frac{\partial - \log p(b)} {\partial |b|^q} = \frac{1}{2} E(\eta^{-1} | b)  = \frac{1}{2} \int{\eta^{-1} p(\eta|b) d \eta}$.
\end{enumerate} \end{theorem}
When as a penalty $-\log p(b)$  is applied to supervised sparse learning,  iterative re-weighted  $\ell_1$  or $\ell_2$ local  methods are
suggested for solving the resulting optimization problem.  We will see that Theorem~\ref{thm:2} shows  the equivalent relationship  
between an iterative re-weighted  method
and an EM algorithm, which is presented in Section~\ref{sec:em}.

\subsection{Duality between Priors and Penalties} \label{sec:duality}

Since there is duality between a prior and a penalty, we are able to construct a penalty
from $p(b)$; in particular, $-\log p(b)$ corresponds to a penalty.  For example,
let $p(b)$ be defined as in (\ref{eqn:eig11}) or (\ref{eqn:eig12}) (see Appendix~\ref{app:cases}).
It is then easily checked that $-\log p(b)$ is concave in $|b|$.
Moreover, if $p(b)$ is given in (\ref{eqn:eig13}), then $-\log p(b)$
induces the $\ell_{1/2}$ penalty $|b|^{1/2}$. In fact, we have the following theorem.

\begin{theorem}  \label{thm:monoto}
Let $p(b)$ be the EP-GIG density given in (\ref{eqn:epgig}). If $-\log p(b)$ is regarded as a function
of $|b^q|$,  then $-\frac{d \log (p(b))}{d |b|^q} $ is completely monotone
on $(0, \infty)$.  Furthermore,  when $0<q\leq 1$,  ${-}\log(p(b))$ is concave in $|b|$ on $(0, \infty)$; namely, ${-}\log(p(b))$
defines a class of nonconvex penalties for $b$.
\end{theorem}

Here a function $\phi(z)$ on $(0, \infty)$ is said to be completely monotone~\citep{FellerBook:1971} 
if it possesses derivatives $\phi^{(n)}$ of all orders and
\[
(-1)^n  \phi^{(n)} (z) \geq 0, \; z>0.
\]
Theorem~\ref{thm:monoto} implies that the first-order and second-order derivatives of   $- \log (p(b))$ with respect to $|b|^q$ 
are nonnegative and nonpositive, respectively.
Thus, ${-}\log(p(b))$ is concave and  nondecreasing in $|b|^q$ on $(0, \infty)$. Additionally,  $|b|^q$ for $0<q\leq 1$ is concave 
in $|b|$ on $(0, \infty)$.  Consequently, when $0<q\leq 1$, ${-}\log(p(b))$ is concave  in  $|b|$ on $(0, \infty)$. In other words,  
${-}\log(p(b))$
with $0<q\leq 1$ induces a nonconvex penalty for $b$.

Figure~\ref{fig:penalty} graphically depicts several penalties, which are obtained from the special priors in Appendix~\ref{app:cases}.
It is readily seen that the fist three penalty functions are concave in $|b|$ on $(0, \infty)$. In Figure~\ref{fig:penalty_others},  
we also  illustrate the penalties
induced from the $1/2$-bridge scale mixture priors (see Examples 7 and 8 in in Appendix~\ref{app:cases}),  generalized  $t$ priors
and  EP-G priors.
Again, we  see that the two penalties
induced from the $1/2$-bridge  mixture priors are concave in $|b|$ on $(0, \infty)$.  This agrees with Theorem~\ref{thm:monoto}.

\begin{figure}[!ht]
\centering
\subfigure[$\gamma=\frac{1}{2}$, $q=1$]{\includegraphics[width=50mm,height=40mm]{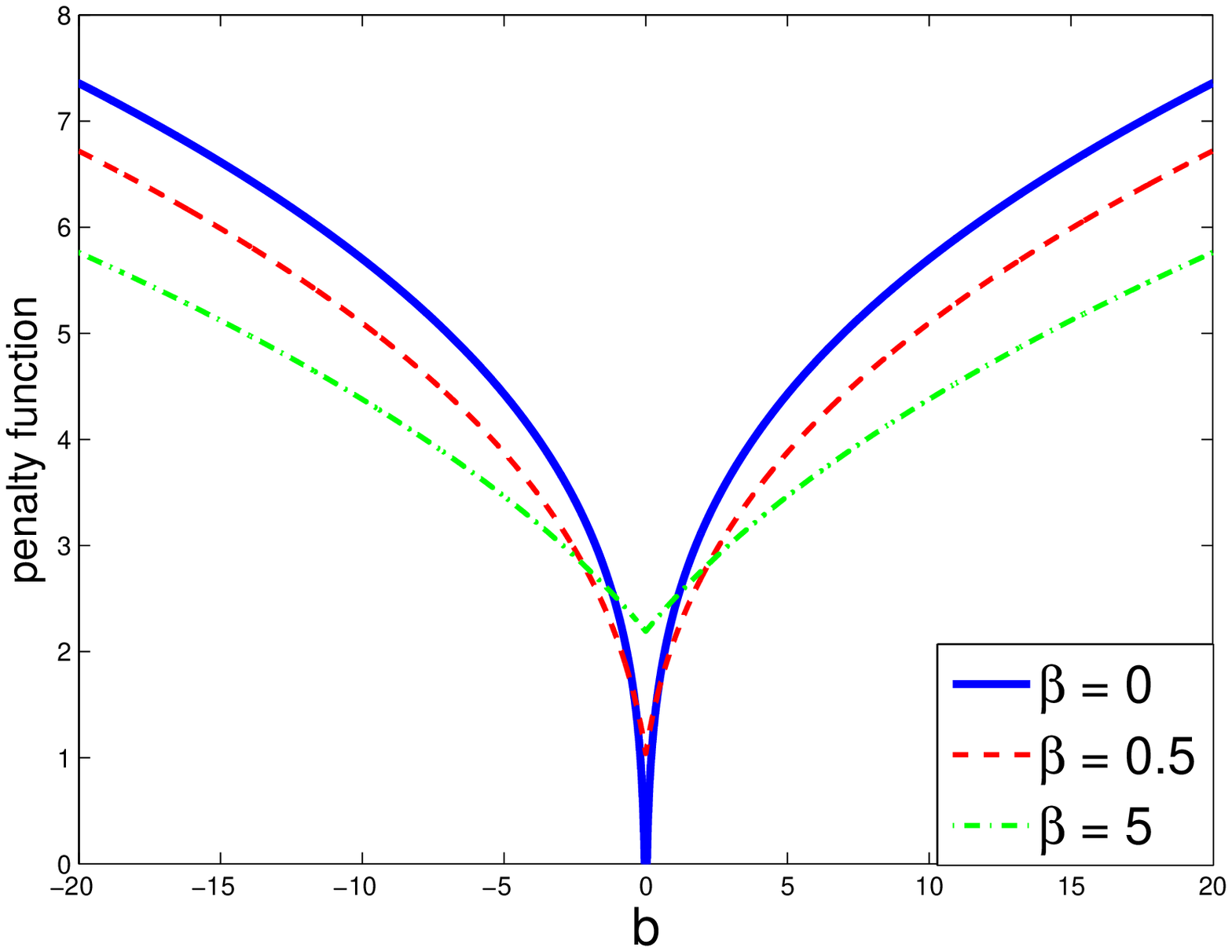}}
\subfigure[$\gamma=\frac{3}{2}$, $q=1$]{\includegraphics[width=50mm,height=40mm]{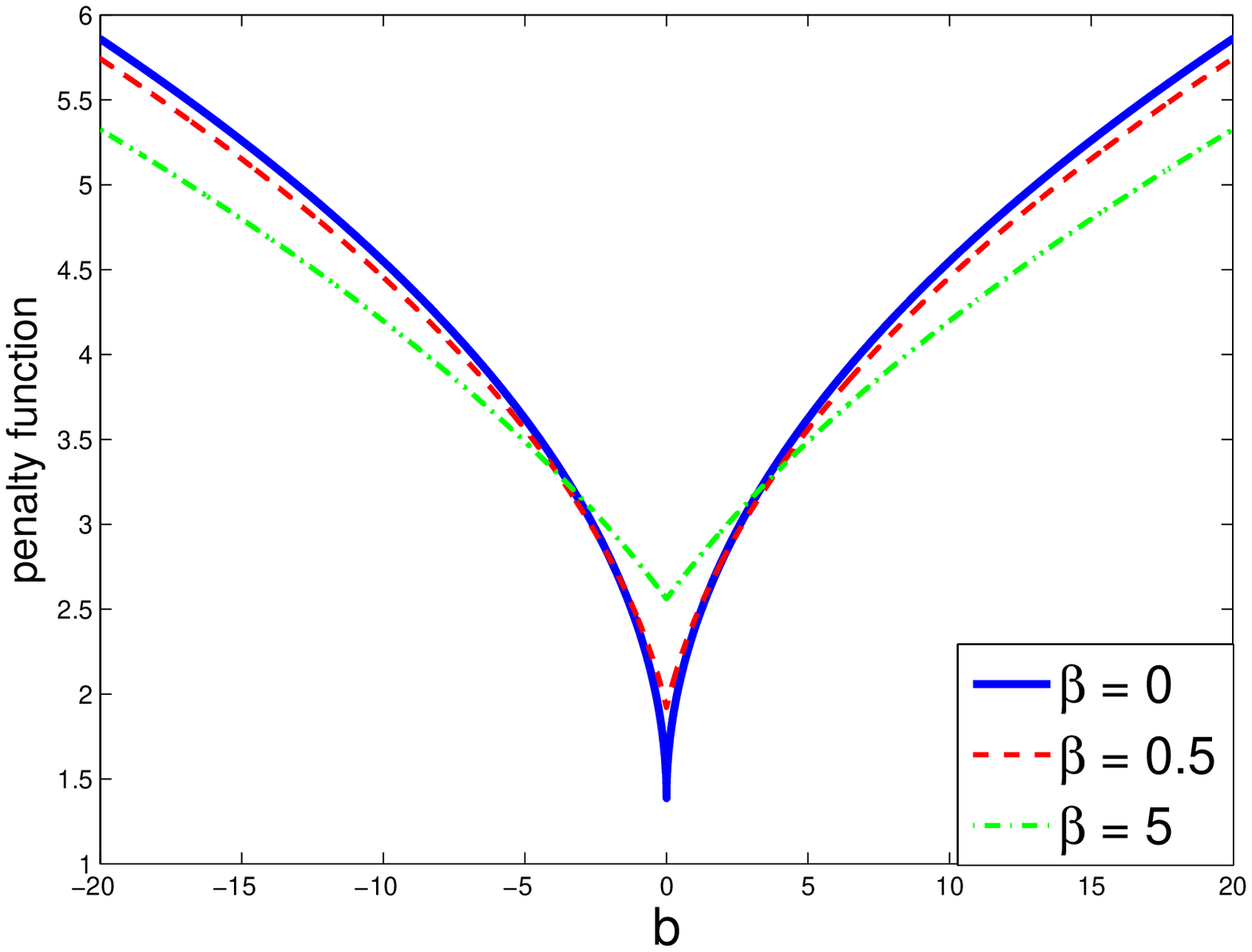}}
\subfigure[$\gamma=-\frac{1}{2}$, $q=1$]{\includegraphics[width=50mm,height=40mm]{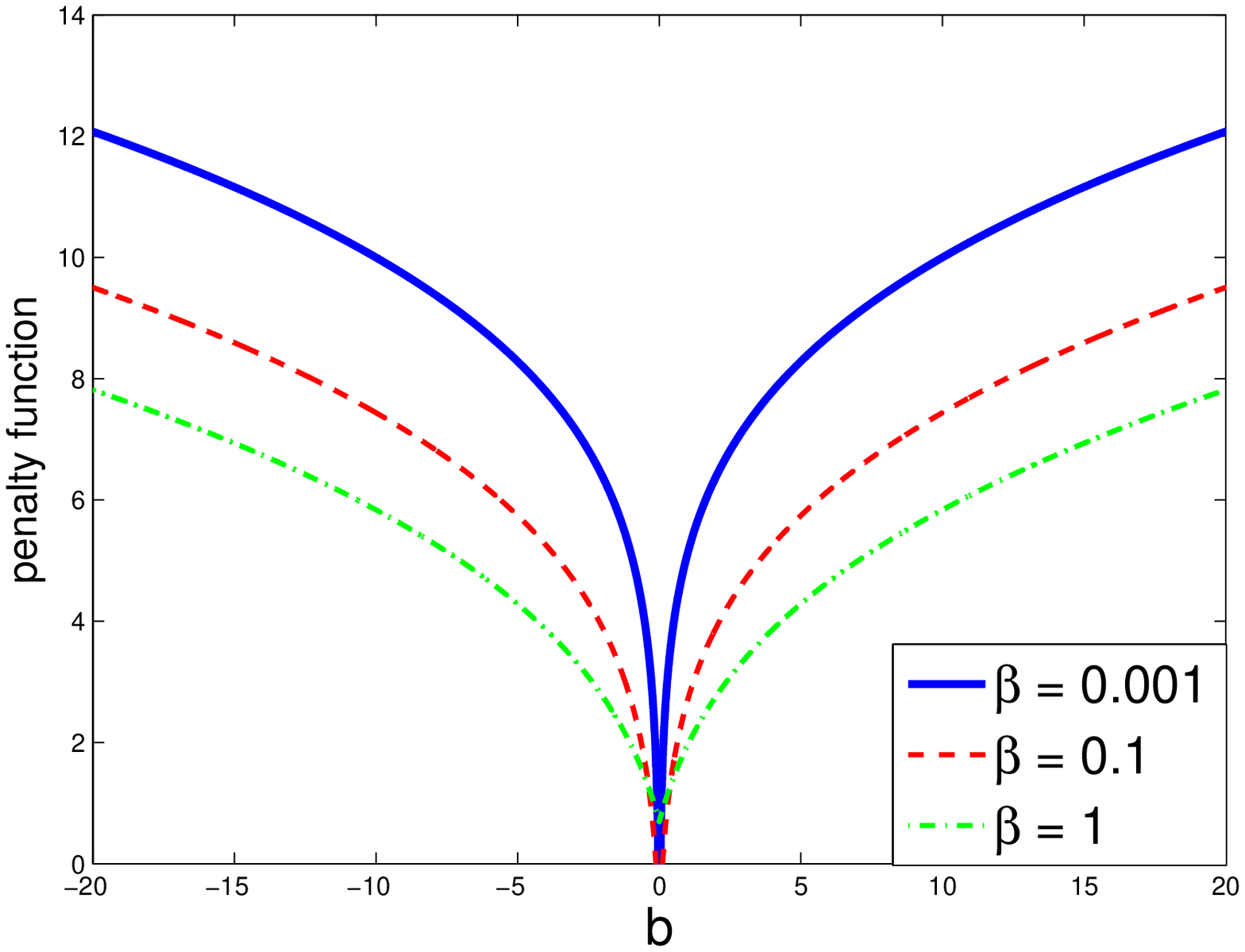}}\\
\subfigure[$\gamma=0$, $q=2$]{\includegraphics[width=50mm,height=40mm]{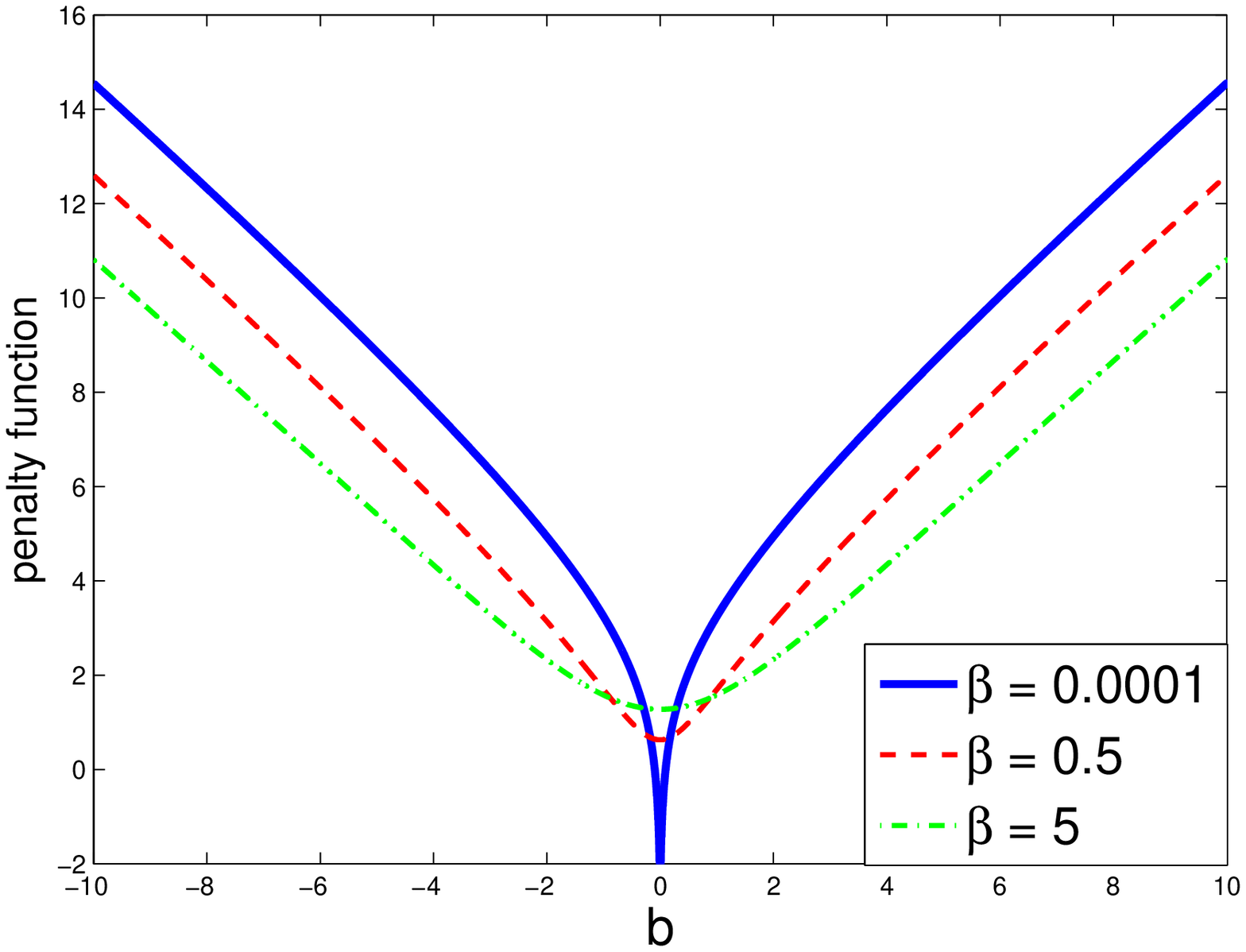}}
\subfigure[$\gamma=1$, $q=2$]{\includegraphics[width=50mm,height=40mm]{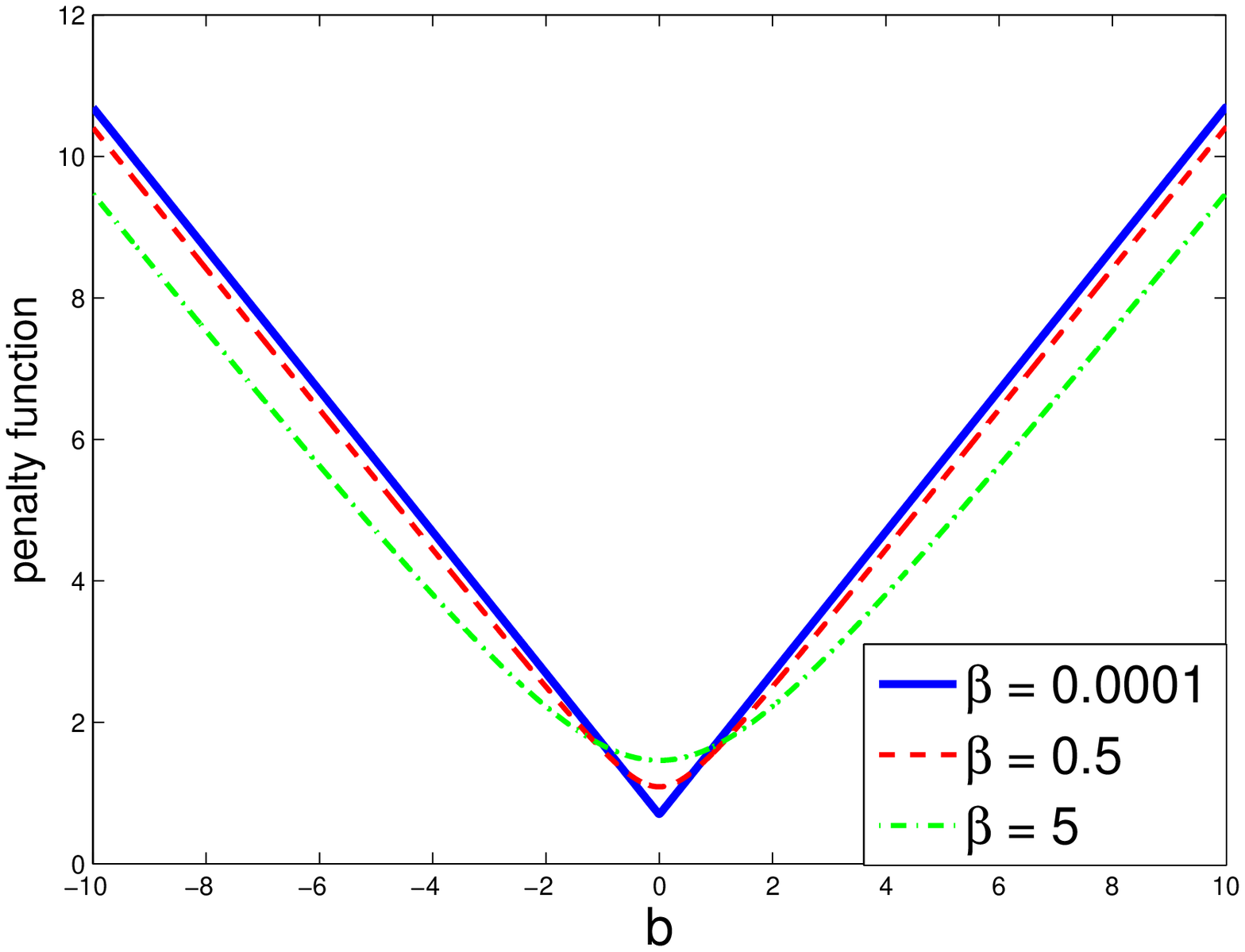}}
\subfigure[$\gamma=-1$, $q=2$]{\includegraphics[width=50mm,height=40mm]{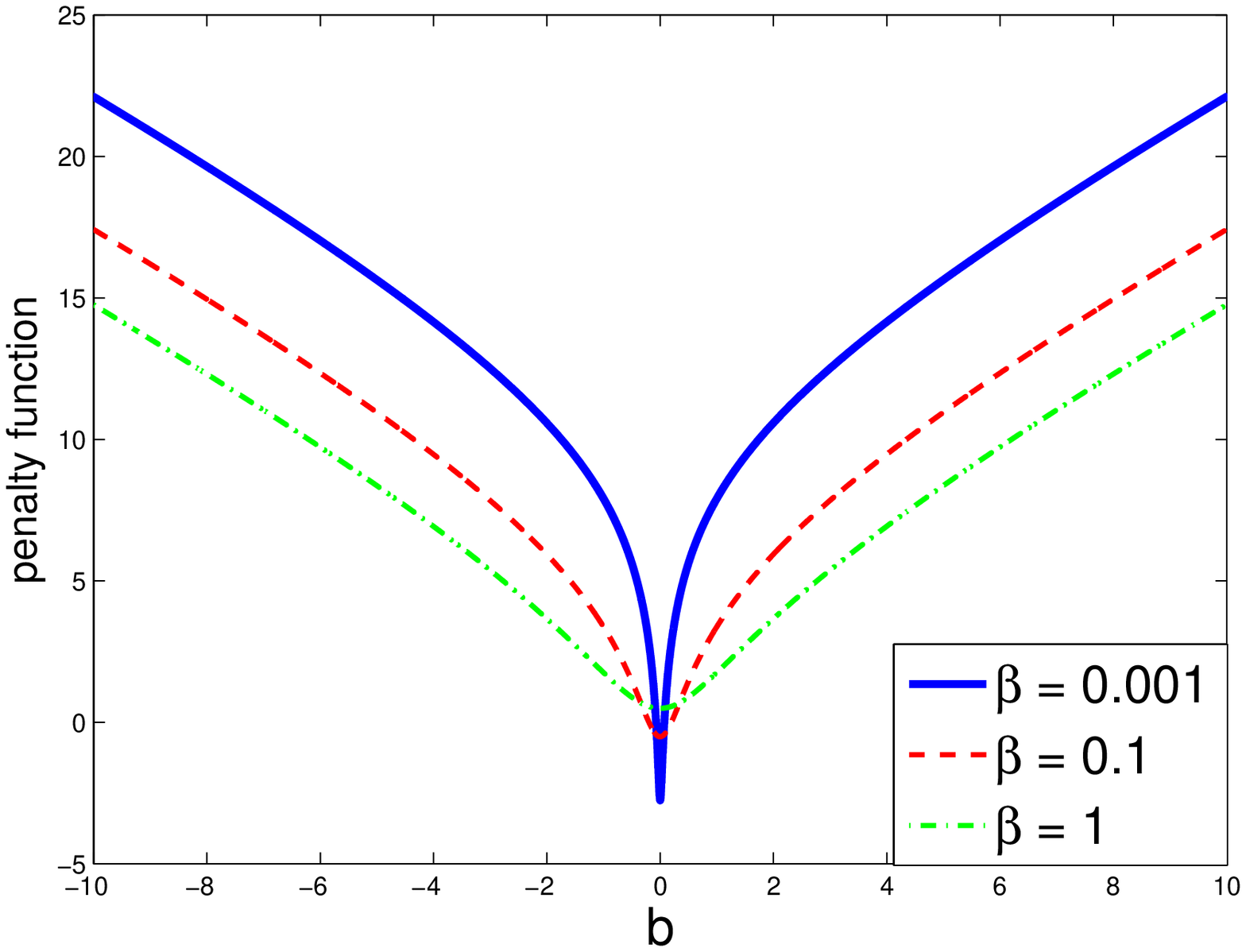}}\\
\caption{Penalty functions induced from exponential power-generalized inverse
gamma (EP-GIG) priors in which $\alpha=1$.} \label{fig:penalty}
\end{figure}

\begin{figure}[!ht]
\centering
\subfigure[EP-GIG$(\gamma=\frac{3}{2}$, $q=\frac{1}{2}, \alpha=1)$]{\includegraphics[width=60mm,height=40mm]{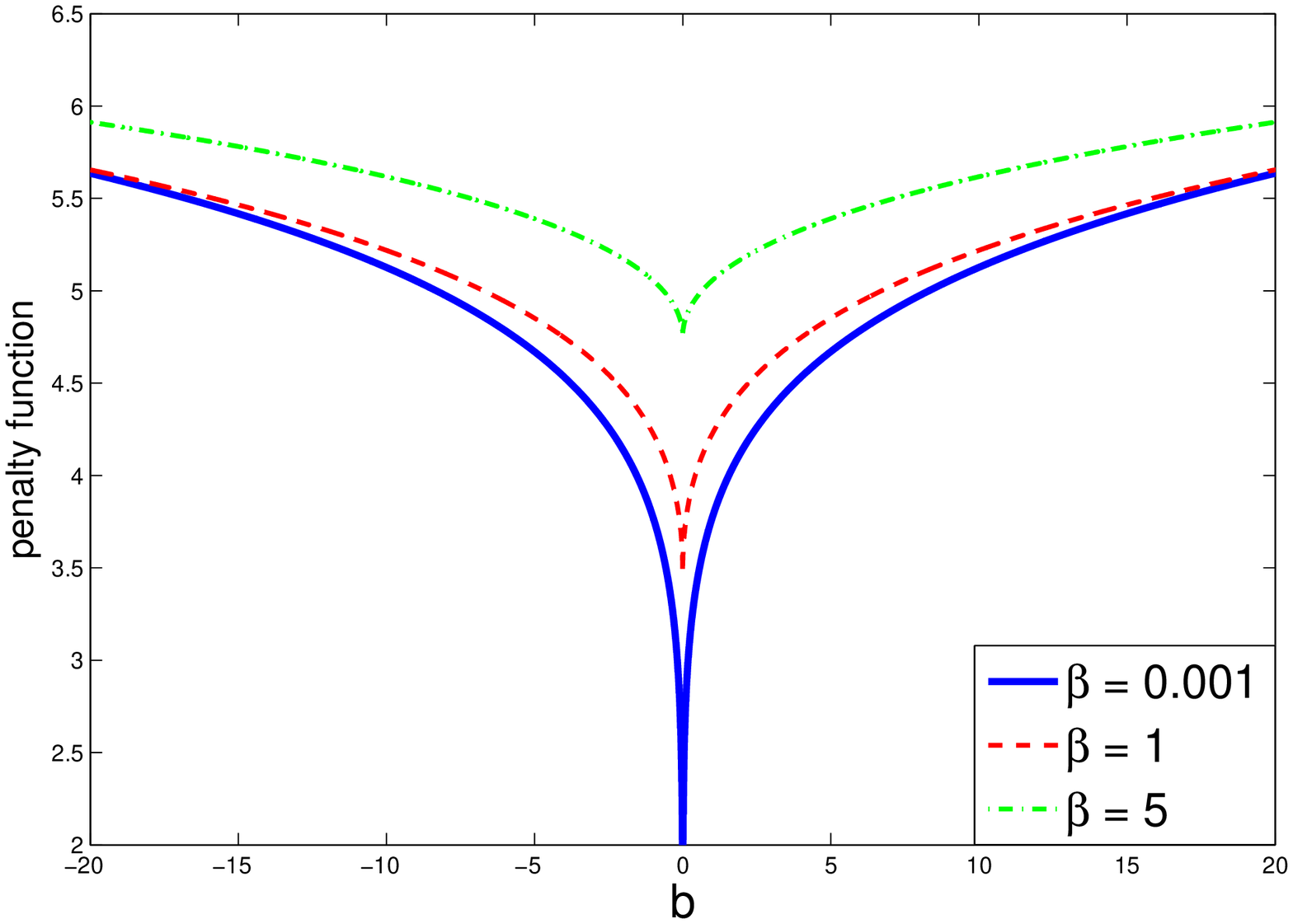}}
\subfigure[EP-GIG$(\gamma=\frac{5}{2}$, $q=\frac{1}{2}, \alpha=1)$]{\includegraphics[width=60mm,height=40mm]{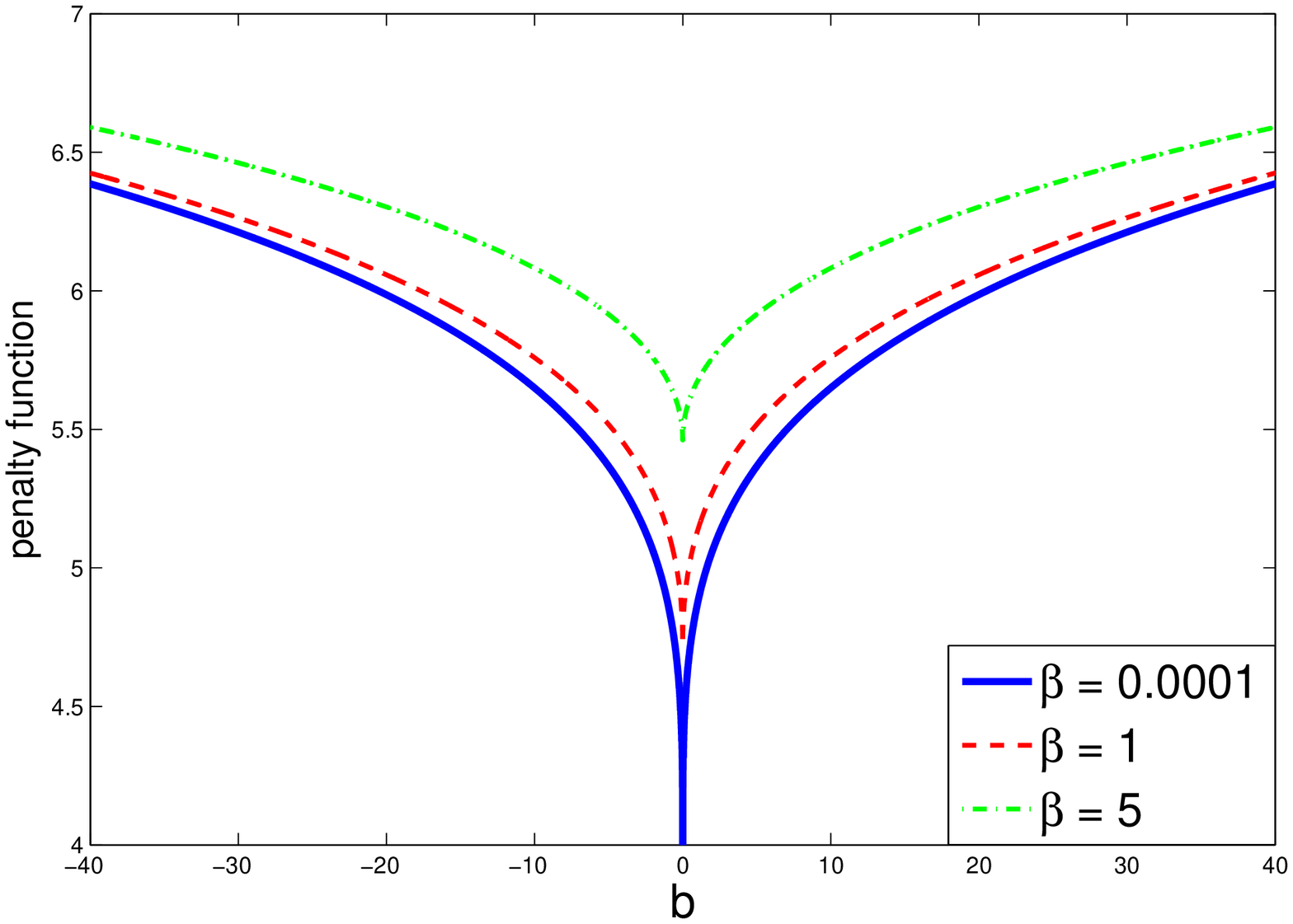}} \\
\subfigure[GT (or GDP) ($q=1$ and $\lambda=1$)]{\includegraphics[width=60mm,height=40mm]{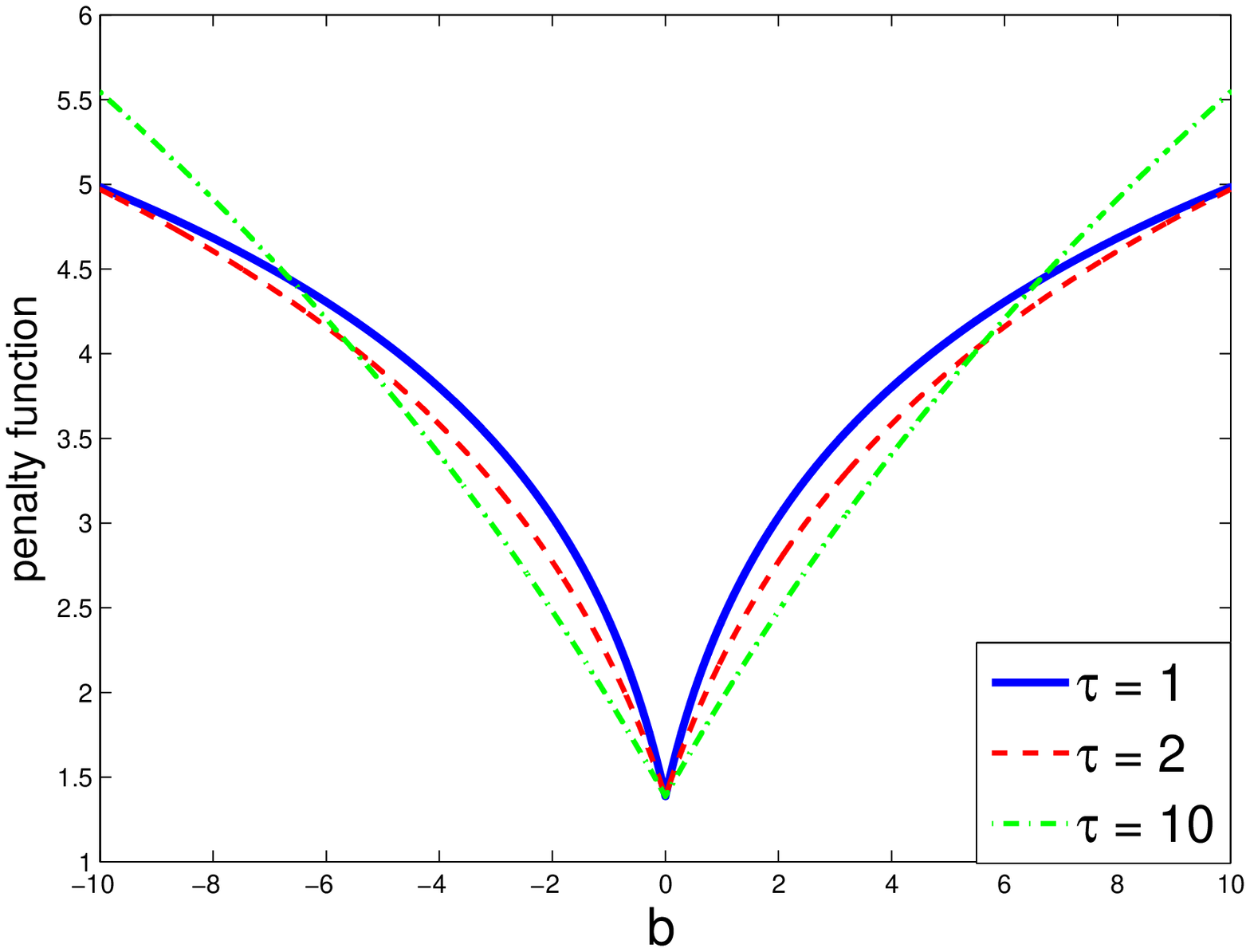}}
\subfigure[GT ($q=2$ and $\lambda =1$)]{\includegraphics[width=60mm,height=40mm]{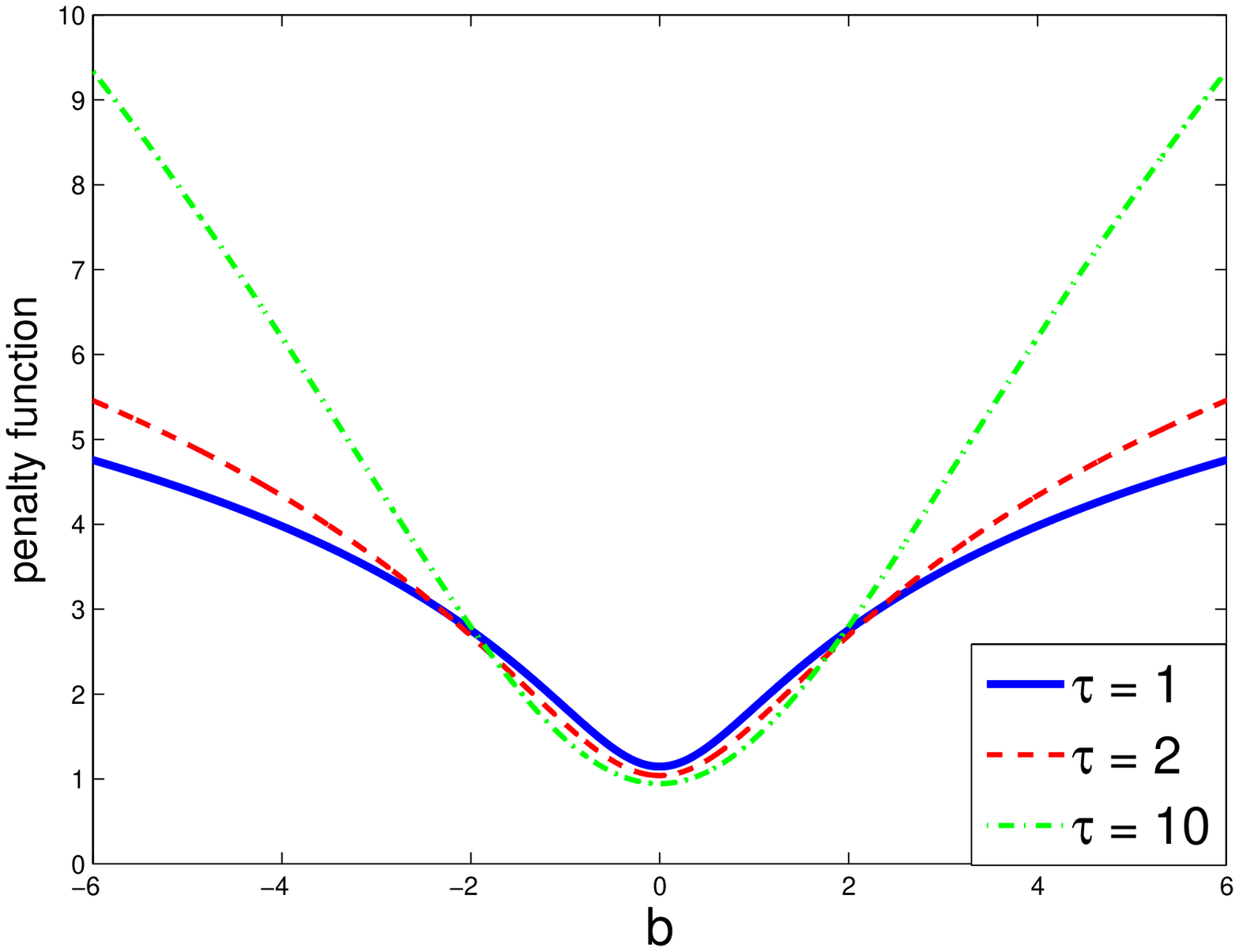}}\\
\subfigure[EG ($q=1$ and $\lambda=1$)]{\includegraphics[width=60mm,height=40mm]{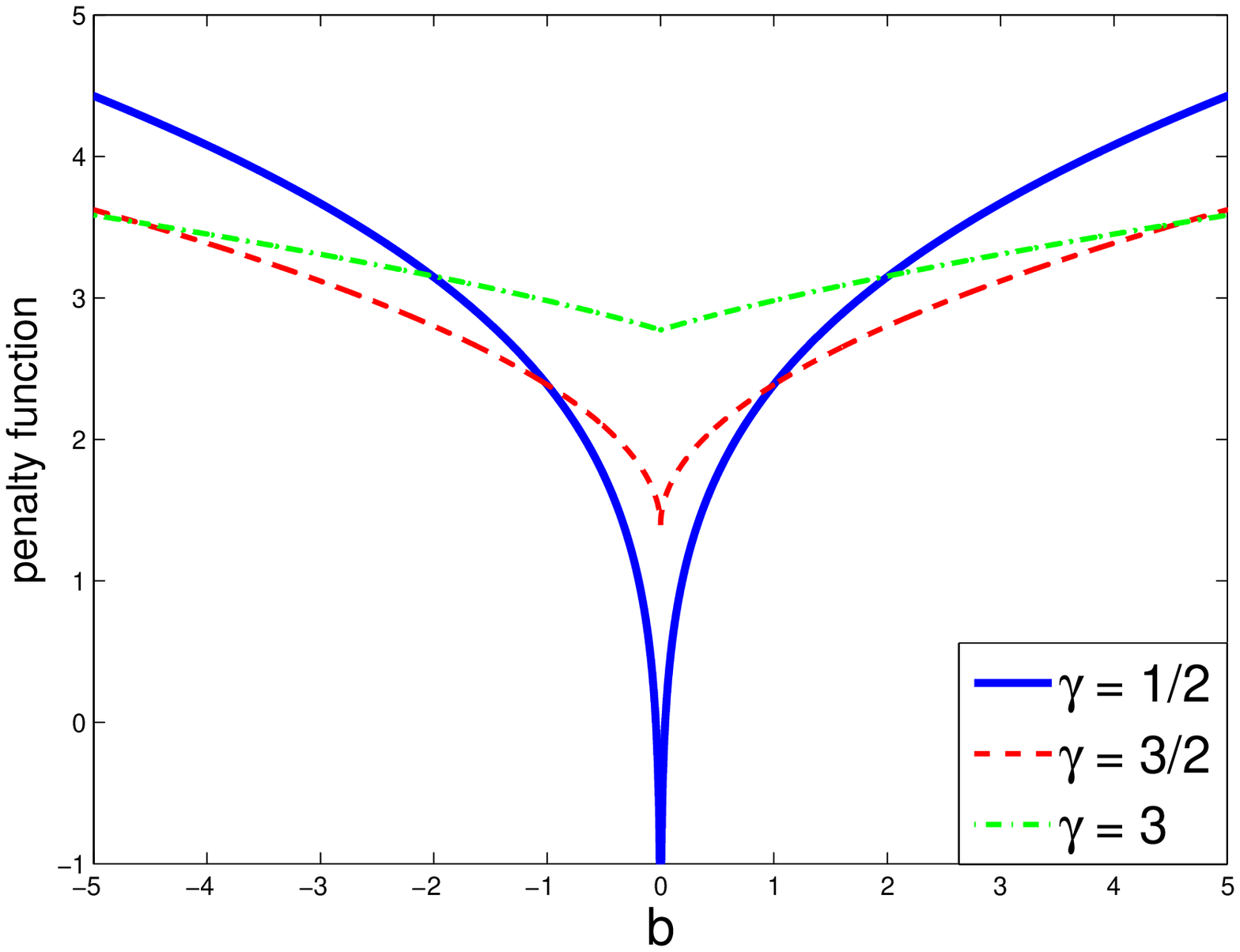}}
\subfigure[EG ($q=2$ and $\lambda=1$)]{\includegraphics[width=60mm,height=40mm]{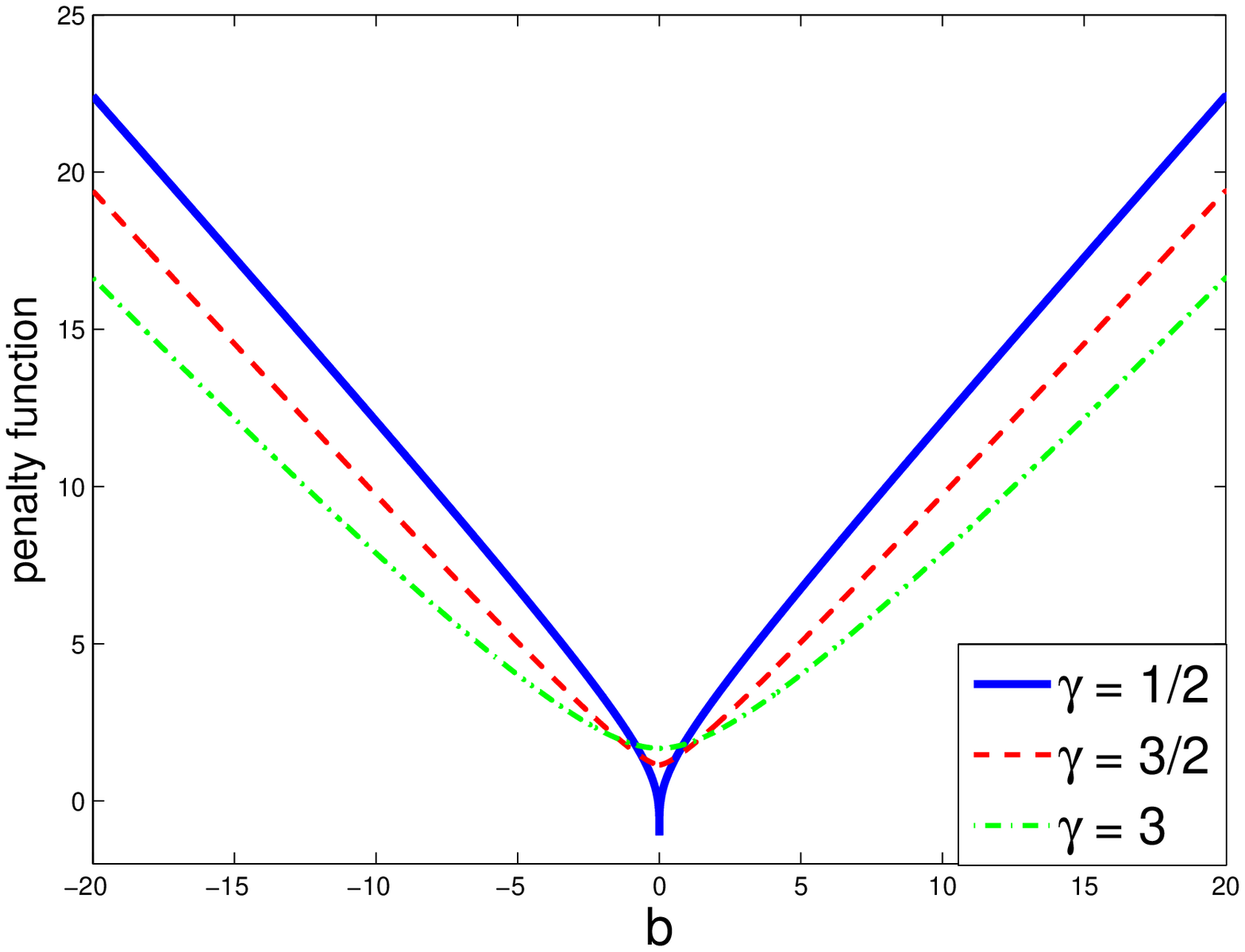}}
\caption{Penalty functions induced from $1/2$-bridge scale mixture priors, exponential power-inverse gamma
(or generalized $t$, GT) priors and exponential power-gamma (EG) priors.}
\label{fig:penalty_others}
\end{figure}

\section{Bayesian Sparse Learning Methods} \label{sec:em}

In this section we apply EP-GIG priors to empirical Bayesian sparse learning.  Suppose we are given a set of training
data $\{(\x_i, y_i): i=1,\ldots, n\}$, where
the $\x_i \in \RB^{p}$ are the input vectors and the $y_i$ are the corresponding
outputs. Moreover, we assume that $\sum_{i=1}^n \x_i=\0$ and $\sum_{i=1}^n y_i=0$. 
We now consider the following linear regression model:
\[
\y = \X\b + \vps,
\]
where $\y=(y_1, \ldots, y_n)^T$ is the $n{\times}1$ output vector,
$\X=[\x_1, \ldots, \x_n]^T$ is the $n {\times} p$
input matrix, and $\vps$ is a Gaussian error vector $N(\vps|\0, \sigma \I_n)$.  We aim to
estimate the vector of regression coefficients $\b=( b_1, \ldots, b_p)^T$ under the MAP framework.

\subsection{Bayesian Sparse Regression} \label{sec:mod1}

We place an EP-GIG prior on each of the elements of $\b$. That is,
\[
p(\b|\sigma) =\prod_{j=1}^p \EIG(b_j|\sigma^{-1}\alpha, \sigma \beta, \gamma, q).
\]
Using the property the the EP-GIG distribution is a scale mixture of
exponential power distributions, we devise an EM algorithm for the MAP estimate of $\b$.
For this purpose, we define a hierarchical model:
\begin{align*}
[\y|\b, \sigma]  & \sim N(\y|\X \b, \sigma \I_n), \\
[b_j | \eta_j, \sigma] & \stackrel{ind}{\sim}  \EP(b_j|0, \sigma \eta_j, q), \\
[\eta_j|\gamma, \beta, \alpha] & \stackrel{iid}{\sim}  \GIG(\eta_j|\gamma, \beta, \alpha), \\
 p(\sigma) & = ``Constant".
\end{align*}
According to Section~\ref{sec:poster}, we have
\[
[\eta_j|b_j, \sigma, \alpha, \beta, \gamma] \sim \GIG\big(\eta_j\big|(\gamma  q{-}1)/q, \; \beta {+} \sigma^{-1}|b_j|^q, \; \alpha\big).
\]

Given the $t$th estimates $(\b^{(t)}, \sigma^{(t)})$ of $(\b, \sigma)$, the E-step of EM calculates
\begin{align*}
Q(\b, \sigma|\b^{(t)}, \sigma^{(t)}) &\triangleq \log p(\y|\b, \sigma) + \sum_{j=1}^p
\int{\log p[b_j | \eta_j, \sigma] p(\eta_j|b_j^{(t)}, \sigma^{(t)}, \alpha, \beta, \gamma)} d \eta_j \\
& \propto  -\frac{n}{2} \log \sigma {-} \frac{1}{2 \sigma}(\y{-}\X \b)^T (\y {-}\X\b) - \frac{p}{q} \log \sigma \\
& \quad - \frac{1}{2 \sigma} \sum_{j=1}^p
|b_j|^q \int \eta_j^{-1} p(\eta_j|b_j^{(t)}, \sigma^{(t)}, \alpha, \beta, \gamma) d \eta_j.
\end{align*}
Here we omit some terms that are independent of
parameters $\sigma$ and $\b$.
In fact, we only need  calculating $E(\eta_j^{-1}|b_j^{(t)}, \sigma^{(t)})$ in the E-step.
It follows from Proposition~\ref{pro:1} (see Appendix~\ref{app:proof}) that
\begin{equation} \label{eqn:w}
w_j^{(t{+}1)} \triangleq  E(\eta_j^{-1}|b_j^{(t)}, \sigma^{(t)}) = \frac{\alpha^{1/2}}
{\big[\beta {+} |b^{(t)}_j|^q /\sigma^{(t)}\big]^{1/2}} \frac{K_{(\gamma  q{-}q{-}1)/q }
\big(\sqrt{\alpha [\beta {+} |b^{(t)}_j|^q /\sigma^{(t)} ]} \big)} {K_{(\gamma  q{-}1)/q} \big(\sqrt{\alpha [\beta {+} |b^{(t)}_j|^q /\sigma^{(t)}]}\big)}.
\end{equation}
There do not exist analytic computational formulae for arbitrary modified Bessel
functions $K_{\nu}$. In this case, we can resort to a numerical approximation
to the Bessel
function. Fortunately, once $\gamma$ and $q$ take the special values in Appendix~\ref{app:cases}, we have  closed-form
expressions for the corresponding Bessel
functions and thus for the $w_j$.
In particular, we have from Proposition~\ref{pro:2} (see Appendix~\ref{app:proof}) that
\[
w_j^{(t{+}1)}  = \left\{ \begin{array}{ll} \Big[\frac{\sigma^{(t)} \alpha} {\sigma^{(t)}\beta + |b_j^{(t)}|^q}\Big]^{1/2}
&  (\gamma q{-}1)/q=1/2, \\
\frac{\sigma^{(t)}+ [\sigma^{(t)}\alpha (\sigma^{(t)}\beta +  |b_j^{(t)}|^q)]^{1/2}} {\sigma^{(t)} \beta + |b_j^{(t)}|^q}
 & (\gamma q{-}1)/q=-1/2, \\
 \frac{3 \sigma^{(t)}}  {\sigma^{(t)} \beta + |b_j^{(t)}|^q}
 + \frac{\sigma^{(t)} \alpha} {\sigma^{(t)}+ [\sigma^{(t)}\alpha (\sigma^{(t)}\beta +  |b_j^{(t)}|^q)]^{1/2}}
 & (\gamma q{-}1)/q=-3/2.  \end{array}  \right.
\]
In Table~\ref{tab:f} we list these cases with different settings of $\gamma$ and $q$.

The M-step maximizes $Q(\b, \sigma|\b^{(t)}, \sigma^{(t)})$ with respect to $(\b, \sigma)$.
In particular, it is obtained as follows:
\begin{align*}
\b^{(t{+}1)} & = \argmin_{\b} \; (\y{-}\X \b)^T (\y {-}\X\b) +  \sum_{j=1}^p w_j^{(t{+}1)} |b_j|^q, \\
\sigma^{(t{+}1)} & = \frac{q}{qn {+} 2 p} \Big\{(\y{-}\X \b^{(t{+}1)})^T (\y {-}\X\b^{(t{+}1)}) + \sum_{j=1}^p w_j^{(t{+}1)} |b_j^{(t{+}1)}|^q \Big\}.
\end{align*}

\begin{table}[t]
\vspace{-0.1in}
\caption{E-steps of EM for different settings of $\gamma$ and $q$. Here we omit superscripts ``(t)".}
\label{tab:f}
\vskip 0.15in
\begin{center}
\begin{small}
\begin{sc}
\begin{tabular}{l|c|c|c|c}
\hline
$(\gamma, q)$ &
$\gamma=\frac{1}{2}, q =1$   & $\gamma=\frac{3}{2}, q =1$  & $\gamma=0, q =2$    & $\gamma=1, q =2$  \\
\hline
$w_j=$   & $\frac{1+ \sqrt{\alpha(\beta + \sigma^{-1} |b_j|)}} {\beta + \sigma^{-1}
|b_j|}$ & $\sqrt{\frac{\alpha} {\beta + \sigma^{-1} |b_j|}}$
& $\frac{1 + \sqrt{\alpha(\beta + \sigma^{-1} b_j^2)}} {\beta + \sigma^{-1} b_j^2}$  & $\sqrt{\frac{\alpha}
{\beta + \sigma^{-1} b_j^2}}$  \\
\hline
\end{tabular}
\end{sc}
\end{small}
\end{center}
\vspace{-0.1in}
\vskip -0.1in
\end{table}

\subsection{The Hierarchy for Grouped Variable Selection} \label{sec:extension}

In the hierarchy specified previously each $b_j$ is assumed to have distinct scale $\eta_j$. We can also let
several $b_j$ share a common scale parameter. Thus we can obtain a Bayesian approach to
group sparsity~\citep{YuanLin:2006}. We next briefly describe this approach.

Let $I_l$ for $l=1, \ldots, g$ be a partition of $I=\{1,2, \ldots, p\}$; that is,
$\cup_{j=1}^g I_j= I$ and $I_j \cap I_l = \emptyset$ for $j\neq l$.  Let $p_l$ be the cardinality of $I_l$, and
$\b_{l}=\{b_j: j \in I_l\}$  denote the subvectors of $\b$, for $l=1, \ldots, g$.
The hierarchy is then specified as
\begin{align*}
[\y|\b, \sigma]  & \sim N(\y|\X \b, \sigma \I_n), \\
[b_j | \eta_l, \sigma] & \stackrel{iid}{\sim}  \EP(b_j|0, \sigma \eta_l, q), j \in I_l \\
[\eta_l|\gamma_l, \beta, \alpha] & \stackrel{ind}{\sim}  \GIG(\eta_l|\gamma_l, \beta, \alpha), l=1, \ldots, g.
\end{align*}
Moreover, given $\sigma$, the $\b_l$ are conditionally independent. By integrating out $\eta_l$,
the marginal density of $\b_l$ conditional
on $\sigma$ is then
\[
p(\b_l|\sigma) =  \frac{K_{\frac{\gamma_l q{-}p_l}{q}}(\sqrt{ \alpha(\beta {+}\sigma^{-1} \|\b_l\|_q^q)})}
{ \big[2^{\frac{q+1}{q}} \sigma^{\frac{1}{q}} \Gamma(\frac{q{+}1} {q})\big]^{p_l} K_{\gamma_l}(\sqrt{\alpha \beta})} \frac{\alpha^{p_l/(2q)}}{\beta^{{\gamma_l}/2} }
\Big[ \beta {+} \sigma^{-1}\|\b_l\|_q^q \Big]^{(\gamma_l q-p_l)/(2q)},
\]
which implies $\b_l$ is non-factorial.  The posterior distribution of $\eta_l$ on $\b_l$ is then
$\GIG(\eta_l|\frac{\gamma_l q {-} p_l}{q}, \; \beta + \sigma^{-1}\|\b_l\|_q^q, \; \alpha)$.

In this case, the iterative procedure for $(\b, \sigma)$ is given by
\begin{align*}
\b^{(t{+}1)} & = \argmin_{\b} \; (\y{-}\X \b)^T (\y {-}\X\b) +  \sum_{l=1}^g w_l^{(t{+}1)} \|\b_l\|_q^q, \\
\sigma^{(t{+}1)} & = \frac{q}{qn {+} 2 p} \Big\{(\y{-}\X \b^{(t{+}1)})^T (\y {-}\X\b^{(t{+}1)}) +
\sum_{l=1}^g w_l^{(t{+}1)} \|\b_l^{(t{+}1)}\|_q^q \Big\},
\end{align*}
where for $l=1, \ldots, g$,
\[
w_l^{(t{+}1)}  = \frac{\alpha^{1/2}} {\big[ \beta + \|\b_l^{(t)}\|_q^q/ \sigma^{(t)}\big]^{1/2}}
\frac{K_{\frac{\gamma_l q {-}  q {-} p_l}{q}}(\sqrt{\alpha [ \beta + \|\b_l^{(t)}\|_q^q/ \sigma^{(t)}]})}
{K_{\frac{\gamma_l q {-} p_l}{q}}(\sqrt{\alpha [\beta + \|\b_l^{(t)}\|_q^q/ \sigma^{(t)}] })}.
\]
Recall that there is usually no analytic computation for $w_l^{(t{+}1)}$.
However, setting such  $\gamma_l$
that $\frac{\gamma_l q {-} p_l}{q}=\frac{1}{2}$ or $\frac{\gamma_l q {-} p_l}{q}=-\frac{1}{2}$ can yield an analytic computation.
As a result, we have
\[
w_j^{(t{+}1)}  = \left\{ \begin{array}{ll} \Big[\frac{\sigma^{(t)} \alpha} {\sigma^{(t)}\beta + \|\b_l^{(t)}\|_q^q}\Big]^{1/2}
&  (\gamma_l q{-}p_l)/q=1/2, \\
\frac{\sigma^{(t)}+ \big[\sigma^{(t)}\alpha (\sigma^{(t)}\beta +  \|\b_l^{(t)}\|_q^q) \big]^{1/2}} {\sigma^{(t)} \beta + \|\b_l^{(t)}\|_q^q}
 & (\gamma_l q{-}p_l)/q=-1/2.  \end{array}  \right.
\]

Figure~\ref{fig:graphal} depicts the hierarchical models in Section~\ref{sec:mod1} and~\ref{sec:extension}. It is clear
that when $g=p$ and $p_1=\cdots=p_g=1$, the models are identical.

\begin{figure}[!ht]
\centering
\subfigure[independent]{\includegraphics[width=65mm, height=65mm]{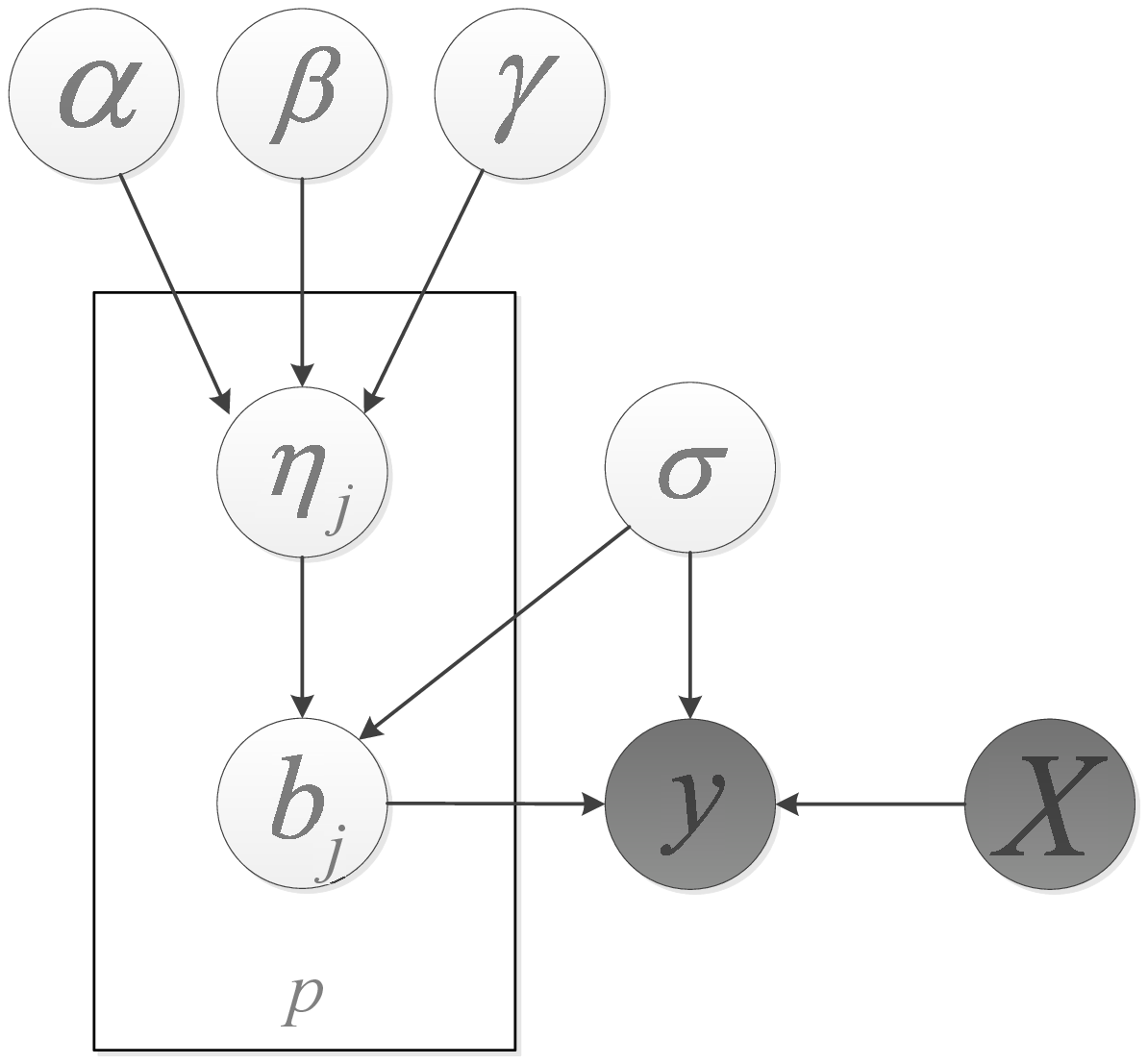}}
\subfigure[grouped]{\includegraphics[width=65mm, height=65mm]{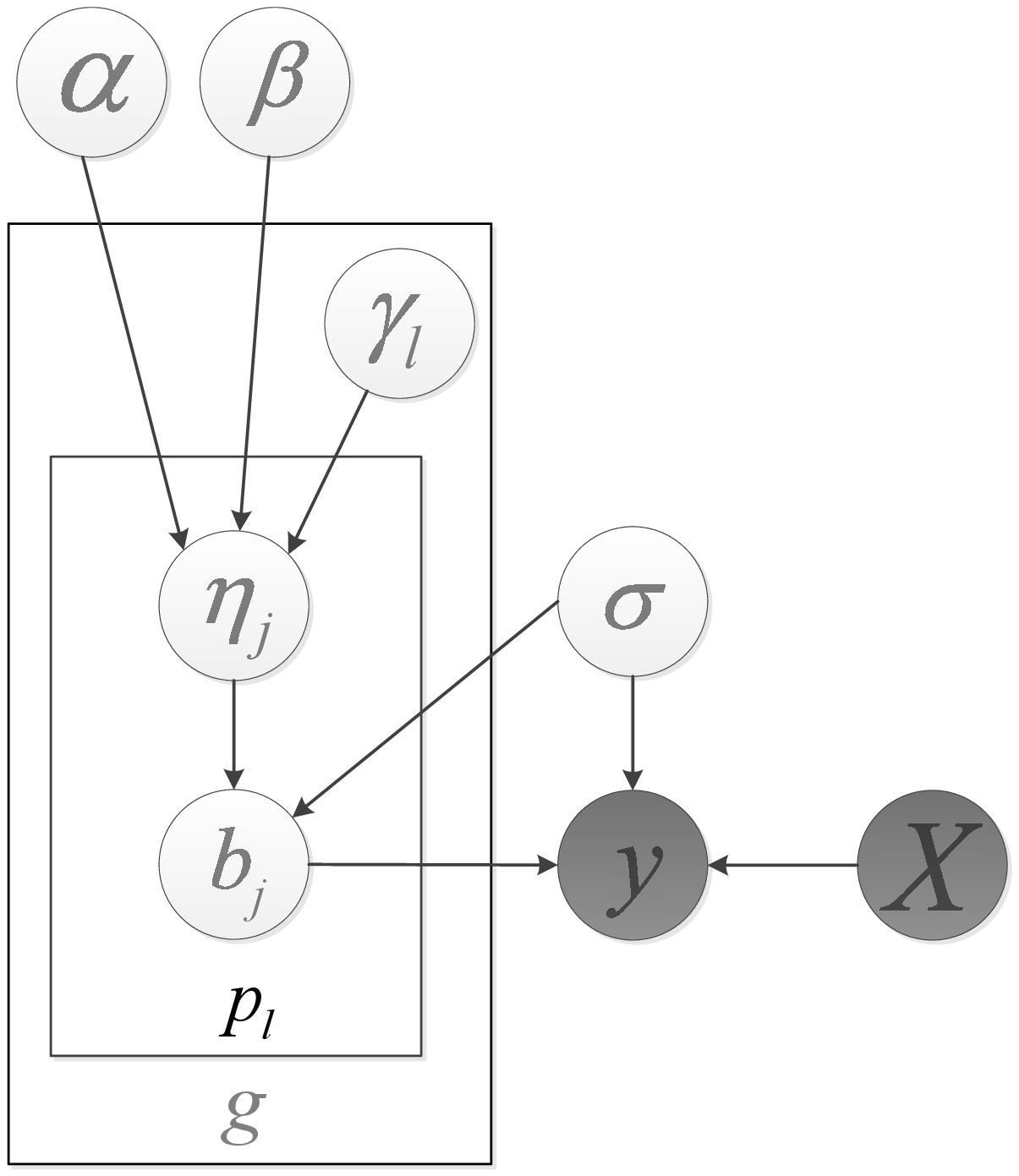}}
\caption{Graphical representations.}
\label{fig:graphal}
\end{figure}

\subsection{Extensions to Logistic Regression} \label{sec:logis}

Another extension is the application in the penalized logistic regression model for classification.
We consider a binary classification problem in which $y \in \{0, 1\}$ now represents the label of the corresponding input vector $\x$.
In the logistic regression model the expected value of $y_i$ is given by
\[P(y_i= 1 |\x_i)= \frac{1}{1+ \exp(- \x_i^T \b)} \triangleq \pi_i.
\]
In this case $\sigma=1$ and the log-likelihood function on the training data becomes
\[
 \log p(\y|\b)=  \sum_{i=1}^n [y_i \log \pi_i+(1{-}y_i) \log(1{-}\pi_i)].
\]

Given the $t$th estimate $\b^{(t)}$ of $\b$, the E-step of EM calculates
\begin{align*}
Q(\b|\b^{(t)}) &\triangleq \log p(\y|\b) + \sum_{j=1}^p
\int{\log p[b_j | \eta_j] p(\eta_j|b_j^{(t)},  \alpha, \beta, \gamma)} d \eta_j \\
& \propto  \sum_{i=1}^n [y_i \log \pi_i+(1{-}y_i) \log(1{-}\pi_i)]  - \frac{1}{2} \sum_{j=1}^p  w_j^{(t{+}1)}
|b_j|^q.
\end{align*}

As for the M-step, a feasible approach is to
first obtain a quadratic approximation to the log-likelihood function
based on its second-order Taylor series expansion at the current estimate $\b^{(t)}$ of the regression vector $\b$. 
We accordingly formulate a penalized linear regression model.
In particular,
the M-step solves the following optimization problem
\[
\b^{(t{+}1)} =\argmin_{\b \in \RB^p} \; (\tilde{\y} {-} \X\b)^{T} \W(\tilde{\y} {-} \X \b) + \sum_{j=1}^p  w_j^{(t{+}1)}
|b_j|^q,
\]
where $\tilde{\y}$, the working response, is defined by
$\tilde{\y} = \X \b^{(t)}+ \W^{-1}(\y-\pii)$,
$\W$ is a diagonal matrix with diagonal elements
$\pi_i(1-\pi_i)$,
and $\pii=(\pi_1, \ldots, \pi_n)^T$. Note that here the $\pi$ are evaluated at $\b^{(t)}$.

\section{Iteratively Re-weighted $\ell_q$ Methods} \label{sec:related}

 We employ a penalty induced from  the EP-GIG prior $\EIG(b|\alpha_0, \beta_0, \gamma, q)$. Then the penalized regression problem
is
\[
\min_{\b}  \;  \frac{1}{2} (\y {-} \X \b)^T (\y {-} \X \b) -  \lambda \sum_{j=1}^p \Big\{ \log K_{\frac{\gamma q {-}1}{q}}
(\sqrt{\alpha_0(\beta_0 {+} |b_j|^q)}) -
\frac{\gamma q {-}1} {2 q} \log(\beta_0 {+} |b_j|^q) \Big\},
\]
which can be solved via the iteratively re-weighted $\ell_q$ method. Given the $t$th estimate $\b^{(t)}$ of $\b$, the method
considers the first Taylor  approximation of $-\log \EIG(b_j|\alpha_0, \beta_0, \gamma, q)$  w.r.t.  $|b_j|^q$
at  $|b_j^{(t)}|^q$  and solves the following problem
\[
\min_{\b}  \;  \frac{1}{2} (\y {-} \X \b)^T (\y {-} \X \b) +  \lambda \sum_{j=1}^p  \omega_j^{(t+1)} |b_j|^q,
\]
where $\omega_j^{(t+1)} = \frac{\partial - \log \EIG(b_j|\alpha_0, \beta_0, \gamma, q)} {\partial |b_j|^q }|_{b_j = b_j^{(t)}} $.
It follows  from Theorem~\ref{thm:2}-(ii) that
\begin{equation} \label{eqn:omega}
\omega_j =  \frac{1}{2} \frac{\sqrt{\alpha_0}}{\sqrt{\beta_0+ |b_j|^q}}
\frac{K_{\frac{\gamma q -1} {q} -1}(\sqrt{\alpha_0(\beta_0+ |b_j|^q)})}
{K_{\frac{\gamma q {-}1} {q} }  (\sqrt{\alpha_0(\beta_0 {+} |b_j|^q)}) }.
\end{equation}

\subsection{Relationship between  EM  and  Iteratively Re-weighted Methods}

We investigate the relationship of the EM algorithm with an iteratively re-weighted $\ell_q$ method where $q=1$
or $q=2$.  When equating $\alpha_0 = \alpha/\sigma$, $\beta_0 = \beta \sigma$
and $\lambda = \sigma$, we immediately see that the $2 \omega_j$ are equal to the $w_j$ in (\ref{eqn:w}).
This implies the iteratively re-weighted minimization method is identical to the EM algorithm given in Section~\ref{sec:mod1}.

When $q=2$, the EM algorithm is identical to the re-weighted $\ell_2$
method and corresponds to a local quadratic approximation~\citep{Fan01,HunterLiAS:2004}. And when $q=1$, the EM algorithm
is the re-weighted $\ell_1$ minimization and corresponds to a local linear approximation~\citep{ZouLi:2008}.
Especially, when we set $\gamma=1$ and $q=2$, the EM algorithm is the same as one
studied by~\cite{Daubechies:2010}.  This implies that the re-weighted $\ell_2$ method
of~\cite{Daubechies:2010} can be equivalently viewed as an EM based on our proposed EP-GIG in Example~5 of Appendix~\ref{app:cases}.
When the EM algorithm is based on our proposed EP-GIG prior in  Example~4 of Appendix~\ref{app:cases}  (i.e. $\gamma=1$ and $q=2$),
it is then the combination of the   re-weighted $\ell_2$ method
of~\cite{Daubechies:2010} and the   re-weighted $\ell_2$ method
of~\cite{ChartrandICASSP:2008}.

When $\gamma=\frac{3}{2}$ and $q=1$, the EM algorithm (see Table~\ref{tab:f})  is equivalent to a  re-weighted $\ell_1$ method, 
which in turn has
a close connection with the re-weighted $\ell_2$ method of~\cite{Daubechies:2010}.
Additionally,  the EM algorithm  based on $\gamma=\frac{1}{2}$ and $q=1$ (see Table~\ref{tab:f})
can be regarded as the combination of the above re-weighted $\ell_1$ method and the re-weighted $\ell_1$
of \cite{CandesWakinBoyd:2008}.  Interestingly, the EM algorithm  based on the EP-GIG priors given in Examples 7 and 8 
of Appendix~\ref{app:cases} (i.e., $\gamma=\frac{3}{2}$ and $q=\frac{1}{2}$ or
$\gamma=\frac{5}{2}$ and $q=\frac{1}{2}$) corresponds a re-weighted
$\ell_{1/2}$ method.

In is also worth mentioning that in Appendix~\ref{app:jeff}  we present EP-Jeffreys priors. Using this prior, 
we can establish the close relationship of
the  adaptive lasso of~\cite{ZouJasa:2006} with an EM algorithm. 	In particular, when $q=1$,  the EM algorithm based on the Jeffreys prior
is equivalent to the adaptive lasso.

\subsection{Oracle Properties}

We now study the oracle property of  our sparse estimator based on Laplace scale mixture priors. For this purpose, following the setup of \cite{ZouLi:2008}, 
 we assume two conditions: (1) $y_i=\x_i^T \b^{*} + \epsilon_i$ where $\epsilon_1, \ldots, \epsilon_n$ are i.i.d errors with mean 0 and variance $\sigma^2$; 
(2) $\X^T \X/n\rightarrow \C$ where $\C$ is a positive definite matrix. Let ${\cal A}=\{j: b_{j}^{*} \neq 0\}$. 
Without loss of generality, we assume that ${\mathcal A}=\{1, 2, \ldots, p_0\}$ with $p_0<p$. Thus, partition $\C$ as
\[
\begin{bmatrix}\C_{11} & \C_{12} \\  \C_{21} & \C_{22} \end{bmatrix}, 
\]
where $\C_{11}$ is $p_0{\times} p_0$. Additionally,  let $\b^{*}_{1} =\{b^{*}_{j}: j \in {\mathcal A}\}$
 and $\b^{*}_{2} = \{u_{nj}:  j \notin {\cal A}\}$.

We in particular consider the following one-step sparse estimator: 
\begin{equation} \label{eqn:llq}
\b_n^{(1)}=\argmin_{\b}  \;  (\y {-} \X \b)^T (\y {-} \X \b) +  \lambda_n \sum_{j=1}^p   |b_j| \frac{Q_{\gamma{-}1} (\alpha_n(\beta_n +|b^{(0)}_j|)) }
{ Q_{\gamma  {-}1} (\alpha_n(\beta_n+ 1) )},
\end{equation}
where $Q_{\nu}(z)=K_{\nu-1}(\sqrt{z})/(\sqrt{z} K_v(\sqrt{z}))$ and $\b^{(0)}=(b_1^{(0)}, \ldots, b_p^{(0)})^T $ is a root-$n$-consistent estimator to $\b^{*}$.  The following theorem shows that
this estimator has the oracle property. That is,  
\begin{theorem}  \label{thm:oracle1} Let $\b_{n1}^{(1)}=\{b_{nj}^{(1)}: j \in \AM\}$  and ${\cal A}_n=\{j: b_{nj}^{(1)} \neq 0\}$. Suppose that
$\lambda_n \rightarrow \infty$, $\lambda_n/\sqrt{n} \rightarrow 0$,  $\alpha_n/{n} \rightarrow c_{1}$ and $\alpha_n\beta_n  \rightarrow c_2$, or that
$\lambda_n /n^{1/4} \rightarrow \infty$, $\lambda_n/\sqrt{n} \rightarrow 0$,  $\alpha_n/\sqrt{n} \rightarrow c_{1}$ and $\alpha_n\beta_n  \rightarrow c_2$. Here $c_1, c_2 \in (0, \infty)$. Then $\b_n^{(1)}$ satisfies the following properties:
\begin{enumerate}
\item[\emph{(1)}]  Consistency in variable selection: $\lim_{n \rightarrow \infty} P({\cal A}_n={\cal A})=1$.
\item[\emph{(2)}]  Asymptotic normality: $\sqrt{n}(\b^{(1)}_{n1} - \b^{*}_{1}) \rightarrow_d N(0, \sigma^2 \C_{11}^{-1})$. 
\end{enumerate} 
\end{theorem}

\section{Experimental Studies} \label{sec:experiment}

In this paper our principal purpose has been to provide a new hierarchical
framework within which we can construct sparsity-inducing priors and EM
algorithms. In this section we conduct an experimental investigation of
particular instances of these EM algorithms.  In particular, we study
the  cases in Table~\ref{tab:f}. We also performed two EM algorithms based on  the generalized $t$ priors,
i.e. the exponential power-inverse gamma priors (see Section~\ref{sec:epig}).
For simplicity of presentation,
we denote them
by ``Method 1," ``Method 2," ``Method 3,"
 ``Method 4,"  ``Method 5,"  ``Method 6," and ``Method 7," respectively.   Table~\ref{tab:methods} lists their EP-GIG prior setups (the notation is the same as in Section~\ref{sec:epgig}).  As we see,  using the EP-GIG priors given in Examples 7 and 8
(see Appendix~\ref{app:cases}) yields  EM algorithms with closed-form E-steps.
However, the corresponding M-steps are a weighted $\ell_{1/2}$ minimization problem, which is not efficiently  solved. Thus,
we did not implement such EM algorithms.

 For ``Method 1," ``Method 2," ``Method 3,"  ``Method 5" and ``Method 6," we fix $\alpha = 1$ and $\sigma^{(0)} = 1$, 
 and use the cross validation
method to select $\beta$.
In ``Method 4" and ``Method 7,"  the parameter $\lambda$ was selected by using  cross validation.   In addition,  
we  implemented the lasso,  the
adaptive lasso (adLasso) and the SCAD-based method  for comparison.  For the lasso, the adLasso
and the re-weighted $\ell_1$ problems in the M-step, we solved the
optimization problems by a coordinate descent algorithm~\citep{MazumderSparsenet:11}.

\begin{table}[!ht]
\caption{The EP-GIG setups of the algorithms.} \label{tab:methods}
\begin{center}
\begin{tabular} {|c|c|c|c|} \hline
  {\tt Method 1}  & {\tt Method 2} & {\tt Method 3} & {\tt Method 4}  \\ \hline
 $\EIG(b|\sigma^{-1}, \sigma \beta,  \frac{1}{2}, 1)$  & $ \EIG(b|\sigma^{-1}, \sigma \beta,  \frac{3}{2}, 1)$ &  $ \EIG(b|\sigma^{-1}, 
 \sigma \beta,  -\frac{1}{2}, 1)$
  &  $\GT(b|0, \frac{\sigma}{\lambda}, \frac{1}{2}, 1)$    \\
 ($q=1$, $\gamma = \frac{1}{2}$) &  ($q=1$, $\gamma = \frac{3}{2}$)  & ($q=1$, $\gamma = -\frac{1}{2}$) &   ($q=1$, $\tau=1$)  \\ \hline \hline
 {\tt Method 5} & {\tt Method 6}  & {\tt Method 7}   &   {\tt AdLasso} \\ \hline
$ \EIG(b|\sigma^{-1}, \sigma \beta,  0, 2)$ & $ \EIG(b|\sigma^{-1}, \sigma \beta,  1, 2)$  & $\GT(b|0, \frac{\sigma}{\lambda}, \frac{1}{2}, 2)$  & $\propto \exp(-|b|^{1/2})$  \\
($q=2$, $\gamma = 0$)  & ($q=2$, $\gamma = 1$)  & ($q=2$, $\tau=1$) & ($q=\frac{1}{2}$)   \\ \hline
 \end{tabular} \end{center}
\end{table}

Recall that ``Method 1,"  ``Method 2,"  ``Method 3,"  ``Method 4" and AdLasso in fact work with the nonconvex penalties.
Especially,  ``Method 1,"  ``Method 2"  and ``Method 3"  are based on the Laplace scale mixture priors proposed  
in Appendix~\ref{app:cases} by us.   ``Method 4" is  based on the GDP prior by~\cite{ArmaganDunsonLee} and~\cite{LeeCaronDoucetHolmes:2010}, 
and we employed the $\ell_{1/2}$
penalty in  the adLasso.  Thus, this adLasso is equivalent to the EM algorithm  which given in Appendix~\ref{app:eig13}.  
Additionally,  ``Method 5"   and ``Method 6" are based on the Gaussian scale mixture priors given  in Appendix~\ref{app:cases} 
by us,  and ``Method 7" is based on the Cauchy prior.
In Appendix~\ref{app:jeff} we present the EM algorithm based on the EP-Jeffreys prior.
This algorithm can be also regarded as an adaptive lasso with weights $1/|b_j^{(t)}|$.
Since the performance of the  algorithms
is same to that of  ``Method 4", here we did not include the results with the  the EP-Jeffreys prior.
We also did not report the results with the Gaussian scale mixture  given in Example~6 of Appendix~\ref{app:cases}, because
they are almost identical to  those with   `Method 5"   or  ``Method 6".

\subsection{Reconstruction on Simulation data} \label{sec:exp1}

We first evaluate the performance of each method on the simulated
data which were used in \cite{Fan01,ZouJasa:2006}.
Let ${\b} = (3, 1.5, 0, 0, 2, 0, 0, 0)^T$, $\x_{i} \stackrel{iid}{\sim} N(\0, \Si)$ with $\Sigma_{i j} = {0.5}^{|i - j|}$,
and $\y_0 = {\X} {\b}$.
Then Gaussian noise ${\epsi} \sim {N}(\0,  \delta^2 \I_n)$ is added to $\y_0$ to form the output vector ${\y} = {\y_0} + {\epsi}$.
Let ${\hat{\b}}$ denote the sparse solution obtained from each method which takes $\X$ and $\y$ as inputs and outputs.
Mean square error (MSE) $\| \y_0 - \X \hat{\b} \|_2^2/n$ is used to measure reconstruction accuracy,
and the number of zeros in $\hat{\b}$  is employed to evaluate variable selection accuracy.
If a method is accurate, the number of ``correct'' © zeros should be $5$ and ``incorrect'' (IC) should be $0$.

For each pair ($n$, $\delta$),
we generate 10,000 datasets. In Table~\ref{tab:toy} we report the numbers of correct and
incorrect zeros as well as the average and standard deviation of MSE on the 10,000 datasets.
From Table~\ref{tab:toy} we see that the nonconvex penalization methods (Methods 1, 2, 3 and 4)
yield the best results in terms of reconstruction accuracy and sparsity recovery.
It should be pointed out that since the weights are defined as the $1/|b_j^{(t)}|^{1/2}$ in the adLasso method, the method
suffers from numerical instability. In addition, Methods 5, 6 and 7 are based on the re-weighted $\ell_2$ minimization,
so they do not naturally produce  sparse estimates.
To achieve  sparseness, they  have to delete small coefficients.

\begin{table}[!ht]\setlength{\tabcolsep}{2.0pt}
\vspace{-0.1in}
\caption{Results on the simulated data sets.}
\label{tab:toy}\vskip 0.05in
\begin{center}
\begin{small}
\begin{sc}
\begin{tabular}{l | c  c  c | c  c  c | c  c  c}
\hline
    &  MSE($\pm$STD) & C &  IC  & MSE ($\pm$STD) & C &  IC & MSE ($\pm$STD) & C &  IC \\
    \hline
    & \multicolumn{3}{c|}{$n = 60$, $\delta = 3$}
    & \multicolumn{3}{c|}{$n = 120$, $\delta = 3$}
    & \multicolumn{3}{c}{$n = 120$, $\delta = 1$}\\
Method 1
        &\textbf{0.699($\pm$0.63)}&4.66&0.08&\textbf{0.279($\pm$0.26)}&4.87&0.01& \textbf{0.0253($\pm$ 0.02)}  & \textbf{5.00}  & 0.00  \\
Method 2
        & 0.700($\pm$0.63)  & 4.55  & 0.07  & 0.287($\pm$0.30)  & 4.83  & 0.02  & 0.0256($\pm$0.03)  & 4.99  & 0.00  \\
Method 3& 0.728($\pm$0.60)  & 4.57  & 0.08  &0.284($\pm$0.28)&\textbf{4.93}&0.00&\textbf{0.0253($\pm$0.02)}&\textbf{5.00}&0.00 \\
Method 4   &0.713($\pm$0.68)&\textbf{4.78}&0.12& 0.281($\pm$0.26)  & 4.89  &0.01&0.0255($\pm$0.03)&\textbf{5.00}&0.00\\
       \hline
Method 5
        & 1.039($\pm$0.56)  & 0.30  & 0.00  & 0.539($\pm$0.28)  & 0.26  & 0.00  & 0.0599($\pm$0.03)  & 0.77  & 0.00  \\
Method 6
        & 0.745($\pm$0.66)  & 1.36  & 0.00  & 0.320($\pm$0.26)  & 1.11  & 0.00  & 0.0262($\pm$0.02)  & 4.96  & 0.00  \\
Method 7
        & 0.791($\pm$0.57)  & 0.20  & 0.00  & 0.321($\pm$0.28)  & 0.42  & 0.00  & 0.0265($\pm$0.02)  & 2.43  & 0.00  \\
        &               &       &       &       &               &       &       &               &       \\
 SCAD    & 0.804($\pm$0.59)  & 3.24  & 0.02  & 0.364($\pm$0.30)  & 3.94  & 0.00  & 0.0264($\pm$0.03) &  4.95 & 0.00  \\
AdLasso & 0.784($\pm$0.57)  & 3.60  & 0.04  & 0.335($\pm$0.27)  & 4.83  & 0.01  & 0.0283($\pm$0.02)  & 4.82  & 0.00  \\
Lasso   & 0.816($\pm$0.53)  & 2.48  & 0.00  & 0.406($\pm$0.26)  & 2.40  & 0.00  & 0.0450($\pm$0.03)  & 2.87  & 0.00  \\
Ridge   & 1.012($\pm$0.50)  & 0.00  & 0.00  & 0.549($\pm$0.27)  & 0.00  & 0.00  & 0.0658($\pm$0.03)  & 0.00  & 0.00  \\
\hline
\end{tabular}
\end{sc}
\end{small}
\end{center}
\vskip -0.1in
\vspace{-0.1in}
\end{table}

\subsection{Regression on Real Data} \label{sec:exp2}

We apply the  methods to linear regression problems and
evaluate their performance on three data sets: Pyrim and Triazines (both obtained from UCI Machine Learning Repository) and the
biscuit dataset (the near-infrared (NIR) spectroscopy of
biscuit doughs)~\citep{Breiman:1997}.
For  Pyrim and Triazines datasets, we randomly held out 70\% of the data for training  and the rest for test.
We repeat this process 10 times, and report the mean and standard deviation of the relative errors defined as
\[
\frac{1}{n_{test}} \sum_{i=1}^{n_{test}} \left|\frac{y(\x_i){-} \tilde{y}(\x_i)}{y(\x_i)}\right|,
\]
where $y(\x_i)$ is the target output of the test input $\x_i$, and $\tilde{y}(\x_i)$ is the prediction value computed from a regression method.
For the NIR dataset, we use the supplied training and test sets: 39 instances
for training and the rest 31 for test~\citep{Breiman:1997}. Since each response of the NIR data includes 4 attributes 
(``fat," ``sucrose," ``flour" and ``water"),
we treat the data as four regression datasets; namely,
the input instances and each-attribute responses constitute one dataset.

The results are listed in Table~\ref{tab:regression}. We see that the
four new methods outperform the adaptive lasso and lasso in most cases.
In particular Methods 1, 2, 3 and 4 (the nonconvex penalization) yield the best
performance over the first two datasets, and Methods 5, 6 and 7 are the best
on the NIR datasets.

\begin{table}[!ht] \setlength{\tabcolsep}{0.9pt}
\caption{Relative error of each method on the three data sets.
The numbers of instances ($n$) and numbers of features ($p$) of each data set are:  $n=74$ and
$p=27$ in Pyrim,  $n=186$ and $p=60$ in Triazines, and  $n=70$ and $p=700$ in NIR.}
\label{tab:regression}\vskip 0.05in
\begin{center}
\begin{small}
\begin{sc}
\begin{tabular}{l | c | c | c | c | c | c }
\hline
                & Pyrim & Triazines & NIR(fat) & NIR(sucrose) & NIR(flour) & NIR(water) \\ \hline
Method 1        & \textbf{0.1342($\pm$0.065)}&0.2786($\pm$0.083)&    0.0530     &     0.0711    &    0.0448     &   0.0305      \\
Method 2        &   0.1363($\pm$0.066)  &\textbf{0.2704($\pm$0.075)}& 0.0556    &     0.0697    &    0.0431     &   0.0312      \\
Method 3      &   0.1423($\pm$0.072)  &   0.2792($\pm$0.081)  &    0.0537     &     0.0803    &    0.0440     &   0.0319          \\
Method 4       &   0.1414($\pm$0.065)  &   0.2772($\pm$0.081)  &    0.0530     &     0.0799    &    0.0448     &   0.0315      \\ \hline
Method 5        &   0.1381($\pm$0.065)  &   0.2917($\pm$0.089)  &    0.0290     &     0.0326    &    0.0341     &   0.0210      \\
Method 6        &   0.2352($\pm$0.261)  &   0.3364($\pm$0.079)  &    0.0299     &\textbf{0.0325} &{0.0341}&\textbf{0.0208}\\
Method 7   &   0.1410($\pm$0.065)  &   0.3109($\pm$0.110)  &\textbf{0.0271}&     0.0423    &\textbf{0.0277}&   0.0279      \\
                &                       &                       &               &               &               &               \\
SCAD            &   0.1419($\pm$0.064)  &   0.2807($\pm$0.079)  &   0.0556      &     0.0715    &    0.0467     &   0.0352      \\
AdLasso         &   0.1430($\pm$0.064)  &  0.2883($\pm$0.080)   &    0.0533     &     0.0803    &    0.0486     &   0.0319      \\
Lasso           &   0.1424($\pm$0.064)  & 0.2804($\pm$0.079)    &    0.0608     &     0.0799    &    0.0527     &   0.0340      \\
\hline
\end{tabular}
\end{sc}
\end{small}
\end{center}
\vspace{-0.1in}
\end{table}

\subsection{Experiments on Group Variable Selection} \label{sec:exp3}

Here we use $p = 32$ with $8$ groups, each of size $4$.
Let $\bet_{1:4} = (3, 1.5, 2, 0.5)^T$,  $\bet_{9:12} = \bet_{17:20} = (6, 3, 4, 1)^T$,
$\bet_{25:28} = (1.5, 0.75, 1, 0.25)^T$ with all other entries set to zero,
while $\X$, $\y_0$, and $\y$ are defined in the same way as in Section~\ref{sec:exp1}.
If a method is accurate, the number of ``correct'' © zeros should be $16$ and ``incorrect'' (IC) should be $0$.
Results are reported in Table~\ref{tab:grouptoy}.

\begin{table}[!t]\setlength{\tabcolsep}{1.5pt}
\vspace{-0.1in}
\caption{Results on the simulated data sets.}
\label{tab:grouptoy}\vskip 0.05in
\begin{center}
\begin{small}
\begin{sc}
\begin{tabular}{l | c  c  c | c  c  c | c  c  c}
\hline
    &  MSE($\pm$STD) & C &  IC  & MSE ($\pm$STD) & C &  IC & MSE ($\pm$STD) & C &  IC \\
    \hline
    & \multicolumn{3}{c|}{$n = 60$, $\delta = 3$}
    & \multicolumn{3}{c|}{$n = 120$, $\delta = 3$}
    & \multicolumn{3}{c}{$n = 120$, $\delta = 1$}\\
Method $1'$
        &2.531($\pm$1.01)&15.85&0.31& 1.201($\pm$0.45)&  \textbf{16.00} & 0.14  & 0.1335($\pm$0.048)  & 15.72 &  0.01 \\
Method $2'$
        &2.516($\pm$1.06)   & 15.87 &0.28&1.200($\pm$0.43)  & 15.97 & 0.10  & 0.1333($\pm$0.047)  & 15.87 &  \textbf{0.00} \\
Method $3'$& 2.445($\pm$0.96)  & \textbf{15.88} & 0.54  & 1.202($\pm$0.43)  & 15.98 &  0.25 & \textbf{0.1301}($\pm$0.047)&\textbf{16.00}& 0.01  \\
Method $4'$
        &  2.674($\pm$1.12) & 15.40 &  0.30 &  1.220($\pm$0.45) & 15.79 &  0.49 & 0.1308($\pm$0.047)  &\textbf{16.00}&  \textbf{0.00} \\ \hline
Method $5'$
        &\textbf{2.314($\pm$0.90)}&5.77&0.04& 1.163($\pm$0.41)  & 7.16  &  0.03 & 0.1324($\pm$0.047)  & \textbf{16.00} & 0.01  \\
Method $6'$
        &2.375($\pm$0.92)&10.18&0.04&\textbf{1.152($\pm$0.41)}  & 15.56 & 0.03  &0.1322($\pm$0.047) &\textbf{16.00}&\textbf{0.00}\\
Method $7'$
        &  2.478($\pm$0.97) &  9.28 &  0.05 &  1.166($\pm$0.41) & 14.17 &  0.03 & 0.1325($\pm$0.047)  & 15.96 &  \textbf{0.00} \\
        &                   &       &       &       &               &       &       &               &       \\
glasso  &2.755($\pm$0.92)&5.52&\textbf{0.00}&1.478($\pm$0.48)&3.45&\textbf{0.00}& 0.1815($\pm$0.058)  & 3.05  &  \textbf{0.00} \\
AdLasso & 3.589($\pm$1.10)  & 11.36 & 2.66  & 1.757($\pm$0.56)  & 11.85 &  1.42 & 0.1712($\pm$0.058)  & 14.09 & 0.32  \\
Lasso   & 3.234($\pm$0.99)  &  9.17 &  1.29 & 1.702($\pm$0.52)  & 8.53  &  0.61 & 0.1969($\pm$0.060)  & 8.03  &  0.05 \\

\hline
\end{tabular}
\end{sc}
\end{small}
\end{center}
\vskip -0.1in
\vspace{-0.1in}
\end{table}

\subsection{Experiments on Classification} \label{sec:exp4}

In this subsection  we apply our hierarchical  penalized logistic regression models
in Section~\ref{sec:logis} to binary classification problems over five
real-world data sets: Ionosphere, Spambase, Sonar, Australian, and Heart from UCI Machine Learning Repository and Statlog.
Table~\ref{tab:data2} gives a brief description of these five datasets.

\begin{table}[!ht]
\caption{The description of datasets.
Here $n$: the numbers of instances; $p$: the numbers of features.} \label{tab:data2}
\begin{center}
\begin{tabular} {|c|c|c|c|c|c|} \hline
 & {\tt Ionosphere}  & {\tt Spambase} & {\tt Sonar} & {\tt Australian} & {\tt Heart} \\ \hline
$n$ & $351$ &  $4601$ & $208$ & $690$  & $270$ \\
$p$& $33$ & $57$ & $60$ & $14$ & $13$ \\ \hline
\end{tabular}
\end{center}
\end{table}

In the experiments, the input matrix $\X \in \RB^{n\times p}$ is normalized such that
$\sum_{i=1}^n x_{i j} = 0$ and $\sum_{i=1}^n x_{i j}^2 = n$ for all $j = 1, \cdots, p$.
For each data set, we randomly choose 70\% for training and the rest for test.
We repeat this process 10 times and report the mean and the standard deviation of classification error rate.
The results in Table~\ref{tab:classification}  are encouraging, because in most cases Methods 1, 2,  3 and 4 based on the nonconvex penalties
perform over the other  methods in both accuracy
and sparsity.

\begin{table}[!ht]\setlength{\tabcolsep}{2.1pt}
\vspace{-0.15in}
\caption{Misclassification rate (\%) of each method on the five data sets.}
\label{tab:classification}
\begin{center}
\begin{sc}
\begin{tabular}{l | c  | c  | c | c  | c }
\hline
            &   {Ionosphere}    &     {Spambase}    &      {Sonar}     &    {Australian}    &      {Heart}      \\
\hline
Method 1    &\textbf{9.91($\pm$2.19)}&7.54($\pm$0.84)&\textbf{18.71($\pm$5.05)}&12.46($\pm$2.08)& 13.83($\pm$3.33)\\
Method 2    &10.19($\pm$2.03)&\textbf{7.47($\pm$0.85)}&19.19($\pm$5.18) & 12.56($\pm$2.06)  & 14.20($\pm$3.50)  \\
Method 3    & 10.00($\pm$1.95)  & 7.58($\pm$0.83)   &  19.03($\pm$4.35) &  12.61($\pm$2.15) & 14.32($\pm$3.60)  \\
Method 4   & 10.66($\pm$1.94)  &  7.61($\pm$0.83)  & 21.65($\pm$5.11)  & 12.65($\pm$2.14)  & 13.95($\pm$3.49)  \\  \hline
Method 5    & 11.51($\pm$3.77)  & 8.78($\pm$0.41)&21.61($\pm$5.70)&\textbf{12.03($\pm$1.74)}&\textbf{13.21($\pm$3.14)}\\
Method 6    & 11.51($\pm$3.72)  &  8.86($\pm$0.41)  & 21.94($\pm$5.85)  & 13.24($\pm$2.22)  & 14.57($\pm$3.38)  \\
Method 7  & 11.70($\pm$4.06)  &  9.49($\pm$0.33)  & 22.58($\pm$5.84)  & 14.11($\pm$2.48)  & 13.46($\pm$3.10)  \\
            &                   &                   &                   &                   &                   \\
SCAD        & 10.47($\pm$2.06)  & 7.58($\pm$0.83    & 21.94($\pm$5.60)  & 12.66($\pm$2.08)  & 13.83($\pm$3.43)  \\
$\ell_{{1}/{2}}$&  10.09($\pm$1.67) &  7.51($\pm$0.86)  & 20.00($\pm$5.95) & 12.56($\pm$2.15)  &  14.20($\pm$3.78) \\
$\ell_1$& 10.47($\pm$1.96)  &  7.57($\pm$0.83)  & 21.61($\pm$5.11)  & 12.66($\pm$2.15)  & 13.95($\pm$3.49)  \\
\hline
\end{tabular}
\end{sc}
\end{center}
\vspace{-0.15in}
\end{table}

\section{Conclusions} \label{sec:concl}

In this paper we have proposed a family of sparsity-inducing priors that we call
\emph{exponential power-generalized inverse Gaussian} (EP-GIG) distributions.
We have defined EP-GIG as a mixture of exponential power distributions with a
generalized inverse Gaussian (GIG) density.  EP-GIG are extensions of Gaussian
scale mixtures and  Laplace scale mixtures. As a special example of the EP-GIG
framework, the mixture of Laplace with GIG can induce a family of nonconvex
penalties. In Appendix~\ref{app:cases}, we have especially  presented five now EP-GIG priors which can induce nonconvex penalties.

Since GIG distributions are conjugate with respect to the exponential
power distribution, EP-GIG are natural for Bayesian sparse learning. In particular,
we have developed hierarchical Bayesian models and devised EM algorithms for
finding sparse solutions.  We have also shown how this framework can be applied
to grouped variable selection and logistic regression problems. Our experiments have  validate that
our proposed EP-GIG priors forming nonconvex penalties  are potentially feasible and effective in sparsity modeling.

%

\appendix

\section{Proofs} \label{app:proof}

In order to obtain proofs, we first present some mathematical preliminaries that will be needed.

\subsection{Mathematical Preliminaries}

The first three of the following lemmas are well known, so we omit their proofs.

\begin{lemma} \label{lem:1}
Let $\lim_{\nu \rightarrow \infty} \; a(\nu) = a$. Then $\lim_{\nu \rightarrow \infty} \; \left(1 + \frac{a(\nu)}{\nu} \right)^{\nu} = \exp(a)$.
\end{lemma}
\begin{lemma}
\emph{(Stirling Formula)}  \label{lem:2}
$\lim_{\nu \rightarrow \infty} \frac{\Gamma(\nu)} {(2\pi)^{1/2} \nu^{\nu-1/2} \exp(-\nu)} =1$.
\end{lemma}
\begin{lemma} Assume $z>0$ and $\nu >0$. Then
  \label{lem:08}
\[ \lim_{\nu \rightarrow \infty} \frac{K_{\nu}(\nu^{1/2} z)} {\pi^{1/2}2^{\nu{-}1/2} \nu^{(\nu-1)/2} z^{-\nu} \exp(-\nu)   \exp(-z^2/4)} =1. \]
\end{lemma}
\begin{proof} Consider the integral representation of $K_{\nu}(\nu^{1/2} z)$ as
\begin{align*}
K_{\nu}(\nu^{1/2} z) &= \pi^{-1/2} 2^{\nu} \nu^{\nu/2} z^{\nu} \Gamma\big(\nu +\frac{1}{2} \big) \int_{0}^{\infty}{ (t^2 + \nu z^2)^{-\nu -\frac{1}{2}}  \cos(t) d t}  \\
    & = \pi^{-1/2} 2^{\nu} \nu^{-(\nu+1)/2} z^{-(\nu{+}1)} \Gamma\big(\nu +\frac{1}{2} \big) \int_{0}^{\infty}
    {\frac{\cos(t)} {(1+ t^2/(\nu z^2))^{\nu +\frac{1}{2}}}  \cos(t) d t}.
\end{align*}
Thus, we have
\begin{align*}
\lim_{\nu \rightarrow \infty} \frac{K_{\nu}(\nu^{1/2} z)}{\pi^{-1/2} 2^{\nu} \nu^{-(\nu+1)/2} z^{-(\nu{+}1)} \Gamma\big(\nu +\frac{1}{2} \big)}
 &=
\lim_{\nu \rightarrow \infty}  \int_{0}^{\infty}{\frac{\cos(t)} {(1+ t^2/(\nu z^2))^{\nu +\frac{1}{2}}}  \cos(t) d t} \\
&= \int_{0}^{\infty} { \cos(t) \exp(-t^2/z^2) d t}.
\end{align*}
We now calculate the integral $\int_{0}^{\infty} { \cos(t) \exp(-t^2/z^2) d t}$ for $z>0$. We denote this integral by $\phi(z)$ and let $u=t/z$. Hence,
\[
 \phi(z)= z \int_0^{\infty} \exp(-u^2)\cos(u z) du =  z  f(z)
\]
where $f(z)= \int_0^{\infty} \exp(-u^2)\cos(u z) du$. Note that
\begin{align*}
f'(z) &= {-} \int_0^{\infty} {\exp(-u^2)\sin(u z) u du} = \frac{1}{2} \int_0^{\infty} { \sin(u z)  d \exp(-u^2)} \\
& = {-} \frac{z}{2} \int_0^{\infty} { \exp(-u^2) \cos(u z)  d u}  = - \frac{z}{2} f(z),
\end{align*}
which implies that $f(z)= C \exp(-z^2/4)$ where $C$ is a constant independent of $z$. 
We   calculate $f(1)$ to obtain $C$. Since
\[
C=\lim_{z\rightarrow +0} f(z)= \lim_{z \rightarrow +0}   \int_{0}^{\infty} e^{-u^2} \cos(u z) d u =\int_{0}^{\infty} e^{-u^2}  d u
= \frac{\sqrt{\pi}}{2},
\]
we have $\phi(z) = \frac{\sqrt{\pi}}{2} z \exp(-z^2/4)$. Subsequently,
\[
\lim_{\nu \rightarrow \infty} \frac{K_{\nu}(\nu^{1/2} z)}{\pi^{-1/2} 2^{\nu} \nu^{-(\nu+1)/2} z^{-(\nu{+}1)} \Gamma\big(\nu +\frac{1}{2} \big)}
=\frac{\sqrt{\pi}}{2} z \exp(-z^2/4).
\]

On the other hand, it follows from Lemmas~\ref{lem:1} and \ref{lem:2}  that
\[
\lim_{\nu \rightarrow \infty} \frac{ \Gamma(\nu+ 1/2)} {(2\pi)^{1/2} \nu^{\nu} \exp(-\nu)} = \lim_{\nu \rightarrow \infty} 
\frac{ \Gamma(\nu+ 1/2)} {\sqrt{2\pi} \nu^{\nu}[1 {+} 1/(2\nu)]^{\nu} \exp({-}\nu) \exp(-1/2)} = 1.
\]
Thus,
\[
\lim_{\nu \rightarrow \infty} \frac{K_{\nu}(\nu^{1/2} z)}{\pi^{\frac{1}{2}} 2^{\nu-\frac{1}{2}} \nu^{\frac{\nu {-} 1}{2}} z^{{-}\nu} \exp(-\nu) \exp(-\frac{z^2} {4})} =1.
\]
\end{proof}

\begin{lemma} \label{lem:3} The modified  Bessel function of the second kind  $K_{\gamma}(u)$  satisfies the following propositions:
\item[\emph{(1)}] $K_{\gamma}(u)= K_{-\gamma}(u)$;
\item[\emph{(2)}] $K_{\gamma+ 1}(u)= 2\frac{\gamma}{u} K_{\gamma}(u) + K_{\gamma{-}1}(u)$;
\item[\emph{(3)}] $K_{1/2}(u)=  K_{-1/2}(u) = \sqrt{\frac{\pi}{2 u}} \exp(-u)$;
\item[\emph{(4)}]  $\frac{\partial K_{\gamma}(u)} {\partial u} = -\frac{1}{2}(K_{\gamma-1}(u) + K_{\gamma+1} (u)) = - K_{\gamma-1}(u) - \frac{\gamma}{u}K_{\gamma}(u)
= \frac{\gamma}{u} K_{\gamma}(u) - K_{\gamma+1}(u)$.
\item[\emph{(5)}] For $\gamma \in (-\infty, +\infty)$, $K_{\gamma} (u) \sim \sqrt{\frac{\pi}{2 u}} \exp(-u)$ as $u \rightarrow +\infty$.
\end{lemma}

\begin{lemma} \label{lem:4}
Let $Q_{\nu}(z) = K_{\nu{-}1}(\sqrt{z}) / (\sqrt{z} K_{\nu}(\sqrt{z}))$ where $\nu \in \RB$ and $z>0$. Then,
$Q_{\nu}$ is completely monotone.
\end{lemma}
\begin{proof} When $\nu\geq 0$, the case was well proved by \cite{Grosswald:1976}. Thus, we only need to prove the case that $\nu<0$. 
In this case,
we let $\nu=-\tau$ where $\tau>0$. Thus,
\[
Q_{\nu} = \frac{K_{-\tau{-}1}(\sqrt{z}) } {\sqrt{z} K_{-\tau}(\sqrt{z})} =  \frac{K_{\tau{+}1}(\sqrt{z}) } {\sqrt{z} K_{\tau}(\sqrt{z})}
= \frac{2 \tau} {z} + \frac{K_{\tau{-}1}(\sqrt{z}) } {\sqrt{z} K_{\tau}(\sqrt{z})},
\]
which is obvious completely monotone.
\end{proof}

The following proposition of the GIG distribution can be  found from~\cite{Jorgensen:1982}.
\begin{proposition} \label{pro:1} Let $\eta$ be distributed according to  ${\GIG}(\eta|\gamma, \beta, \alpha)$ with $\alpha>0$
and $\beta>0$. Then
\[
E(\eta^{\nu}) = \Big(\frac{\beta}{\alpha} \Big)^{\nu/2} \frac{K_{\gamma+\nu}(\sqrt{\alpha \beta})}{K_{\gamma}(\sqrt{\alpha \beta})}.
\]
\end{proposition}

We are especially interested in the  cases that $\gamma=1/2$, $\gamma=-1/2$, $\gamma=3/2$ and $\gamma=-3/2$. For these cases,
we have the following results.
\begin{proposition} \label{pro:2} Let $\alpha>0$
and $\beta>0$.
\begin{enumerate}
\item[\emph{(1)}] If $\eta$ is distributed according to  ${\GIG}(\eta|1/2, \beta, \alpha)$, then
\[
E(\eta) = \frac{1+\sqrt{\alpha \beta}} {\alpha}, \quad  E(\eta^{-1}) = \sqrt{\frac{\alpha}{\beta}}.
\]
\item[\emph{(2)}] If $\eta$ is distributed according to  ${\GIG}(\eta|{-}1/2, \beta, \alpha)$, then
\[
E(\eta)  = \sqrt{\frac{\beta}{\alpha}}, \quad  E(\eta^{-1}) = \frac{1+\sqrt{\alpha \beta}} {\beta}.
\]
\item[\emph{(3)}] If $\eta$ is distributed according to  ${\GIG}(\eta|3/2, \beta, \alpha)$, then
\[
E(\eta) = \frac{3}{\alpha} +  \frac{\beta}{1+\sqrt{\alpha \beta}} , \quad  E(\eta^{-1}) =  \frac{\alpha}{1+\sqrt{\alpha \beta}}.
\]
\item[\emph{(4)}] If $\eta$ is distributed according to  ${\GIG}(\eta|{-}3/2, \beta, \alpha)$, then
\[
E(\eta)  = \frac{\beta}{1+\sqrt{\alpha \beta}},  \quad  E(\eta^{-1}) =  \frac{3}{\beta} + \frac{\alpha}{1+\sqrt{\alpha \beta}} .
\]
\end{enumerate}
\end{proposition}

\begin{proof}
It follows from  Lemma~\ref{lem:3} that
$K_{3/2}(u)= \frac{1{+}u}{u} K_{1/2}(u) = \frac{1{+}u}{u} K_{-1/2}(u)$.

We first consider the case that $\eta \sim \GIG(\eta|1/2, \beta, \alpha)$. Consequently, $E(\eta^{-1})= \alpha/\beta$ and
\[
E(\eta)= \Big(\frac{\beta}{\alpha} \Big)^{1/2} \frac{K_{\frac{3}{2}}(\sqrt{\alpha \beta})}{K_{\frac{1}{2}}(\sqrt{\alpha \beta})}
= \Big(\frac{\beta}{\alpha} \Big)^{1/2} \frac{1+ \sqrt{\alpha \beta}}{\sqrt{\alpha \beta} } = \frac{1+\sqrt{\alpha \beta} }{\alpha}.
\]

As for the case that $\eta \sim \GIG(\eta|{-}3/2, \beta, \alpha)$, it follows from Proposition~\ref{pro:1} that
\[
E(\eta) =  \Big(\frac{\beta}{\alpha}\Big)^{1/2} \frac{K_{-1/2}(\sqrt{\alpha \beta})} {K_{-3/2}(\sqrt{\alpha \beta})} 
= \frac{\beta}{1+\sqrt{\alpha \beta}}
\]
and
\[
E(\eta^{-1}) =  \Big(\frac{\beta}{\alpha}\Big)^{-1/2} \frac{K_{-5/2}(\sqrt{\alpha \beta})} {K_{-3/2}(\sqrt{\alpha \beta})}  
= \frac{3}{\beta} + \frac{\alpha}{1+\sqrt{\alpha \beta}}.
\]

Likewise, we have the second and third parts.
\end{proof}

\subsection{The Proof of Proposition~\ref{pro:7}}

Note that
\begin{align*}
\lim_{\tau \rightarrow \infty} G(\eta|\tau, \tau \lambda) &= \lim_{\tau \rightarrow \infty} \frac{(\tau \lambda)^{\tau}}
{\Gamma(\tau)} \eta^{\tau-1} \exp(-\tau \lambda \eta) \\
& = \lim_{\tau \rightarrow \infty}  \frac{(\tau \lambda)^{\tau}}{(2\pi)^{\frac{1}{2}}  \tau^{\tau-\frac{1}{2}} \exp(-\tau)} 
\eta^{\tau-1} \exp(-\tau \lambda \eta)
\quad \mbox{(Use the Stirling Formula)} \\
& = \lim_{\tau \rightarrow \infty} \frac{\tau^{\frac{1}{2}}} {(2\pi)^{\frac{1}{2}} \eta}  \frac{(\lambda \eta)^{\tau}} 
{\exp((\lambda \eta-1) \tau)}.
\end{align*}
Since $\ln u \leq u -1$ for $u>0$, with equality if and only if $u=1$, we can obtain the proof.

Likewise, we also have the second part.

\subsection{The Proof of Proposition~\ref{pro:08}}

Using Lemma~\ref{lem:08}
\begin{align*}
\lim_{\gamma \rightarrow +\infty}\GIG(\eta|\gamma, \beta, \gamma \alpha) &
= \lim_{\gamma \rightarrow +\infty} \frac{\gamma^{\gamma/2} (\alpha/\beta)^{\gamma/2}}{2 K_{\gamma}(\sqrt{\gamma \alpha \beta})} 
\eta^{\gamma-1} \exp(-(\gamma \alpha \eta + \beta \eta^{-1})/2) \\
&= \lim_{\nu \rightarrow +\infty} \frac{ \alpha^{\gamma} \exp(\frac{\alpha \beta} {4}) \exp(-\beta \eta^{-1}/2)} {\pi^{\frac{1}{2}} 2^{\gamma{+}\frac{1}{2}} \gamma^{-\frac{1}{2}}
} \eta^{\gamma-1} \exp(-\gamma(\alpha \eta/2 {-} 1)) \\
&= \lim_{\gamma \rightarrow +\infty} \frac{\eta^{-1} \gamma^{\frac{1}{2}} \exp(\frac{\alpha \beta} {4})}{ (2\pi)^{\frac{1}{2}} 
\exp(\beta \eta^{-1}/2) } (\alpha \eta/2)^{\gamma} \exp(-\gamma(\alpha \eta/2 {-} 1)) \\
& = \delta(\eta|2/\alpha).
\end{align*}
Again since $\ln u \leq u -1$ for $u>0$, with equality if and only if $u=1$, we can obtain the proof of Part (1).

Let $\tau=-\gamma$. We have
\begin{align*}
\lim_{\gamma \rightarrow -\infty} \GIG(\eta|\gamma,  -\gamma \beta,  \alpha)
&= \lim_{\tau \rightarrow +\infty} \GIG(\eta|-\tau,  \tau \beta, \alpha)  \\
&= \lim_{\tau \rightarrow +\infty} \frac{ (\alpha/(\tau \beta))^{-\tau/2}}{2 K_{\tau}(\sqrt{\tau \alpha \beta})} \eta^{-\tau-1} 
\exp(-(\alpha \eta + \tau \beta \eta^{-1})/2)
\end{align*}
due to $K_{-\tau}(\sqrt{\tau \alpha \beta}) = K_{\tau}(\sqrt{\tau \alpha \beta})$. Accordingly, we also have the second part.

\subsection{The Proof of Proposition~\ref{pro:09}}

Based on (\ref{eqn:gig2}) and Lemma~\ref{lem:3},  we have that
\begin{align*}
\lim_{\psi \rightarrow +\infty} p(\eta) = \lim_{\psi \rightarrow + \infty} \frac{\psi^{1/2}}{\sqrt{2 \pi}} \frac{1} {\exp(\frac{\psi}{2 \phi \eta} (\phi \eta-1)^2)}  
=\delta(\eta|\phi).
\end{align*}

\subsection{The Proof of Theorem~\ref{thm:1}}

\begin{align*}
p(b) &= \int_{0}^{+\infty}{{\EP}(b|0, \eta, q) {\GIG}(\eta|\gamma, \beta, \alpha) d \eta} \\
     & = \int_{0}^{+\infty}{\frac{\eta^{-1/q} }{2^{\frac{q+1}{q}} \Gamma(\frac{q{+}1} {q})} \exp(-\eta^{-1} |b|^q/2) \frac{(\alpha/\beta)^{{\gamma}/2}}{2 K_{\gamma}(\sqrt{\alpha \beta})}  \eta^{{\gamma}-1} \exp(-(\alpha \eta + \beta \eta^{-1})/2) d \eta} \\
     & = \frac{(\alpha/\beta)^{{\gamma}/2}}{2^{\frac{2q+1}{q}} \Gamma(\frac{q{+}1} {q}) K_{\gamma}(\sqrt{\alpha \beta})}
     \int_{0}^{+\infty}{ \eta^{\frac{\gamma q- 1}{q}-1} \exp[-(\alpha  \eta + (\beta + |b|^q) \eta^{-1})/2] d \eta} \\
     & = \frac{K_{\frac{\gamma q{-}1}{q}}(\sqrt{ \alpha(\beta {+}|b|^q)})} {2^{\frac{q+1}{q}} \Gamma(\frac{q{+}1} {q}) K_{\gamma}(\sqrt{\alpha \beta})} \frac{\alpha^{1/(2q)}}{\beta^{{\gamma}/2} } [\beta {+}|b|^q]^{(\gamma q-1)/(2q)}.
\end{align*}

\subsection{The Proof of Theorem~\ref{thm:02}}

 The proof can be directly obtained from Proposition~\ref{pro:08}. That is,
\begin{align*}
\lim_{\gamma \rightarrow +\infty}  \EIG(b|\gamma \alpha, \beta, \gamma, q) &= \lim_{\gamma \rightarrow +\infty}
\int{\EP(b|0,  \gamma, q) \GIG(\eta|\gamma, \beta, \gamma \alpha) d \eta}  \\
&= \int{\EP(b|0,  \gamma, q) \delta(\eta| 2/\alpha) d \eta} =\EP(b|0, 2/\alpha, q).
\end{align*}
Likewise, we have the proof for the second part.  As for the third part, it can be immediately obtained from
Proposition~\ref{pro:09}.

\subsection{The Proof of Theorem~\ref{thm:5}}

Note that
\[
\lim_{\tau \rightarrow \infty} \GT(b|u, \tau/\lambda, \tau/2, q) = \lim_{\tau \rightarrow \infty} \frac{q}{2} \frac{\Gamma(\frac{\tau}{2} {+}\frac{1}{q})} {\Gamma(\frac{\tau}{2}) \Gamma(\frac{1}{q})}  \Big(\frac{\lambda}{\tau}\Big)^{\frac{1}{q}}
\Big(1 + \frac{\lambda}{\tau} |b-u|^q \Big)^{-(\frac{\tau}{2}+\frac{1}{q})}.
\]

It follows from  Lemmas~\ref{lem:1} and \ref{lem:2}  that
\[
\lim_{\tau \rightarrow \infty}  \; \Big(1 + \frac{\lambda}{\tau} |b-u|^q \Big)^{-(\frac{\tau}{2}+\frac{1}{q})} = \exp\left(-\frac{\lambda}{2} |b-u|^q\right)
\]
and
\begin{align*}
\lim_{\tau \rightarrow \infty} \frac{q}{2} \frac{\Gamma(\frac{\tau}{2} {+}\frac{1}{q})} {\Gamma(\frac{\tau}{2}) \Gamma(\frac{1}{q})}  \Big(\frac{\lambda}{\tau}\Big)^{\frac{1}{q}}
& = \lim_{\tau \rightarrow \infty} \frac{q}{2  \Gamma({1}/{q}) }   \Big(\frac{\lambda}{\tau}\Big)^{\frac{1}{q}}
\frac{( \frac{\tau}{2} {+}\frac{1}{q})^{\frac{\tau}{2} {+}\frac{1}{q}-\frac{1}{2}}}  { (\frac{\tau}{2})^{\frac{\tau-1}{2} } }
\exp\Big(-\frac{1}{q} \Big)
\\
&=  \lim_{\tau \rightarrow \infty} \frac{q}{2  \Gamma({1}/{q}) }
\Big(\frac{\lambda}{2} + \frac{\lambda}{q \tau} \Big)^{1/q} \Big(1+ \frac{2}{q \tau} \Big)^{\frac{\tau-1}{2} }  \exp\Big(-\frac{1}{q} \Big)
\\
&= \frac{q}{2  \Gamma({1}/{q}) }
\Big(\frac{\lambda}{2}  \Big)^{1/q}.
\end{align*}
Thus, the proof completes.

\subsection{The Proof of Theorem~\ref{thm:6}}

According to  Proposition~\ref{pro:7}, we have
\begin{align*}
\lim_{\gamma \rightarrow \infty} \EG(b|\lambda \gamma, \gamma/2, q)  &=\lim_{\gamma \rightarrow \infty} \int_{0}^{\infty}{\EP(b|0, \eta, q)
   G(\eta|\gamma/2, \lambda \gamma/2)} d \eta  \\
 &=  \int_{0}^{\infty}{\EP(b|0, \eta, q)
  \big[\lim_{\gamma \rightarrow \infty} G(\eta|\gamma/2, \lambda \gamma/2)} \big] d \eta \\
  &=    \int_{0}^{\infty}{ \EP(b|0, \eta, q)  \delta(\eta|1/\lambda) }d \eta \\
  &= \EP(b|0, 1/\lambda, q).
 \end{align*}

\subsection{The Proof of Theorem~\ref{thm:00}}

With the setting that  $\gamma=\frac{1}{2} +\frac{1}{q}$, we have
\begin{align*}
\int_{0}^{\infty}{\EP(b|0, \eta, q)    G(\eta|\gamma, \alpha/2)} d \eta
&=   \frac{\alpha^{\frac{1}{q}  + \frac{1}{4}}
|b|^{\frac{q}{4} } }  {2^{\frac{2}{q} {+}\frac{1}{2} }  \Gamma(\frac{q{+}1}{q})  \Gamma(\frac{1}{2} +\frac{1}{q})}
K_{1/2} (\sqrt{\alpha |b|^q }) \\
& = \frac{\alpha^{\frac{1}{q}  + \frac{1}{4}}
|b|^{\frac{q}{4} } }  {2^{\frac{2}{q} {+}\frac{1}{2} }   2^{-\frac{2}{q}} \sqrt{\pi}  \frac{2}{q} \Gamma(\frac{2}{q}) }
 \frac{2^{-1/2} \sqrt{\pi}}{ (\alpha |b|^q)^{1/4}} \exp(-\sqrt{\alpha |b|^q })
\\
& = \frac{q \alpha^{1/q}} {4  \Gamma(\frac{2}{q}) } \exp(-\sqrt{\alpha |b|^q }) =\EP(b|0, \alpha^{-1/2}/2, q/2) .
\end{align*}
Here we use the fact that $\Gamma(\frac{q{+}1}{q})  \Gamma(\frac{1}{2} {+}\frac{1}{q}) = 2^{1- 2(\frac{1}{2} {+} \frac{1}{q})}
\sqrt{\pi} \Gamma(1+ \frac{2}{q})
= 2^{-\frac{2}{q}} \sqrt{\pi} \frac{2}{q} \Gamma(\frac{2}{q}) $.


\subsection{The Proof of Theorem~\ref{thm:2}}

The first part is immediate. We consider the proof of the second part.
It follows from Lemma~\ref{lem:3} that
\begin{align*}
\frac{\partial -\log p(b)} {\partial |b|^q}  &= \frac{K_{\frac{\gamma q -1} {q} -1}(\sqrt{\alpha(\beta+ |b|^q)})  +  \frac{(\gamma q -1)/q} { \sqrt{\alpha(\beta+ |b|^q)}}  K_{\frac{\gamma q -1} {q} }  (\sqrt{\alpha(\beta+ |b|^q)})}  {K_{\frac{\gamma q -1} {q} }
 (\sqrt{\alpha(\beta+ |b|^q)}) }   \frac{1}{2} \frac{\alpha}{\sqrt{\alpha(\beta+ |b|^q)}} \\
& \quad - \frac{\gamma q -1}{2 q} \frac{1}{\beta+ |b|^q} \\
& = \frac{1}{2} \frac{\sqrt{\alpha}}{\sqrt{\beta+ |b|^q}}    \frac{K_{\frac{\gamma q -1} {q} -1}(\sqrt{\alpha(\beta+ |b|^q)})}
 {K_{\frac{\gamma q -1} {q} }  (\sqrt{\alpha(\beta+ |b|^q)}) }
= \frac{1}{2} E(\eta^{-1}| b).
\end{align*}
due to that $\eta|b \sim \GIG(\eta|(\gamma q -1)/q, \sqrt{\beta+|b|^q}, \alpha)$.

\subsection{The Proof of Theorem~\ref{thm:monoto}}

For notational simplicity, we let $z=|b^q|$, $\nu=\frac{\gamma q -1} {q}$ and $\phi(z) = \frac{\partial -\log p(b)} {\partial |b|^q}$.
According to the above proof, we have
\[
\phi(z) = \frac{\alpha}{2} \frac{1}{\sqrt{\alpha(\beta+ z)}}    \frac{K_{ \nu-1}(\sqrt{\alpha(\beta+ z)})}  {K_{\nu}
(\sqrt{\alpha(\beta+ z)}) }.
\]
It then follows from Lemma~\ref{lem:4}
that $\phi(z)$ is completely monotone.

\subsection{The Proof of Theorem~\ref{thm:oracle1}}

Let ${\b}_n^{(1)} =\b^{*} + \frac{\u}{\sqrt{n}}$ and
\[
\hat{\u} = \argmin_{\u} \; \bigg\{\Psi(\u) := \Big\|\y - \X (\b^{*} +\frac{\u}{\sqrt{n}} ) \Big\|^2  
  + \lambda_n \sum_{j=1}^p \omega_j^{(0)} |b^{*}_j {+} \frac{u_j}{\sqrt{n}}| \bigg\},
\]
where
\[\omega_j^{(0)} =  \frac{\sqrt{\alpha_n\beta_n + \alpha_n}}{\sqrt{ \alpha_n(\beta_n+ |b^{(0)}_j|) }}
\frac{ K_{\gamma  {-}2 } (\sqrt{\alpha_n(\beta_n+ |b^{(0)}_j|)})}
{K_{\gamma  {-}1}    (\sqrt{\alpha_n(\beta_n {+} |b^{(0)}_j|)}) }  \frac{ K_{\gamma  {-}1 } (\sqrt{\alpha_n(\beta_n+ 1))}}
{K_{\gamma  {-}2}    (\sqrt{\alpha_n(\beta_n {+} 1) )}}.
\]
Consider that
\[
 \Psi(\u) - \Psi(0) 
= \u^T(\frac{1}{n}\X^T \X) \u {-} 2 \frac{\epsi^T \X}{\sqrt{n}} \u {+} \lambda_n \sum_{j=1}^p
\omega_j^{(0)} \Big \{ \big|b^{*}_j {+} \frac{u_j}{\sqrt{n}} \big| {-} |b^{*}_j| \Big\}.
\]
We know that $\X^T\X/n \rightarrow \C$ and $\frac{ \X^T \epsi}{\sqrt{n}} \rightarrow_{d} N(\0, \sigma^2 \C)$. We thus only consider
the third term of the right-hand side of the above equation. Since $\alpha_n \beta_n \rightarrow c_1$ and $\alpha_n \rightarrow \infty$ (note that $\alpha_n/{n} \rightarrow c_2 >0$ implies $\alpha_n \rightarrow +\infty$), we have
\[
 \frac{ K_{\gamma  {-}1 } (\sqrt{\alpha_n(\beta_n+ 1))}}
{K_{\gamma  {-}2}    (\sqrt{\alpha_n(\beta_n {+} 1) )}} \rightarrow 1.
\]
If $b^{*}_j=0$, then $\sqrt{n}(|b^{*}_j+\frac{u_j}{\sqrt{n}}|-|b^{*}_j|)=|u_j|$. And since $\sqrt{n} b_j^{(0)}=O_p(1)$, we have $\alpha_n |b^{(0)}_j|
= (\alpha_n/\sqrt{n}) \sqrt{n} |b^{(0)}_j| =O_p(1)$. Hence, $Q_{\gamma{-}1} (\alpha_n(\beta_n +|b^{(0)}_j|))$ converges to a positive constant in probability. 
As a result,
we obtain
\[
\frac{\lambda_n \omega^{(0)}_j}{\sqrt{n}}   \rightarrow_p   \rightarrow \infty.
\]
due to
\[
\frac{\sqrt{\alpha_n\beta_n + \alpha_n}}{\sqrt{n} } \frac{K_{\gamma{-}1} (\sqrt{\alpha_n\beta_n + \alpha_n})}{K_{\gamma{-}2} (\sqrt{\alpha_n\beta_n + \alpha_n})} \rightarrow \sqrt{c_2}.
\]
If $b^{*}_j\neq0$, then $\omega^{(0)}_j\rightarrow_p \frac{1} {\sqrt{|b^{(0)}_j|}} >0$ and  $\sqrt{n}(|b^{*}_j+\frac{u_j}{\sqrt{n}}|-|b^{*}_j|)\rightarrow u_j \sgn(b^{*}_j)$. 
Thus $\lambda_n \frac{\omega^{(0)}_j}{\sqrt{n}}\sqrt{n}(|b^{*}_j {+} \frac{u_j}{\sqrt{n}}| - |b^{*}_j|)\rightarrow_p 0$.
The remaining parts of the proof  can be immediately obtained via some slight modifications to that in \cite{ZouJasa:2006} or
\cite{ZouLi:2008}.  We here omit them.

\section{Several Special EP-GIG Distributions} \label{app:cases}

We now present eight other important concrete EP-GIG distributions, obtained
from particular settings of ${\gamma}$ and $q$.

\paragraph{Example 1}   We first discuss the case that $q=1$ and ${\gamma}=1/2$. That is, we employ the mixing distribution 
of $L(b|0, \eta)$
with ${\GIG}(\eta|1/2, \beta, \alpha)$. In this case, since
\[
K_{\frac{1}{2}{-}1}(\sqrt{ \alpha(\beta {+}|b|)}) = K_{-1/2}(\sqrt{\alpha(\beta {+}|b|)}) =
\frac{(\pi/2)^{1/2}}{( \alpha(\beta {+}|b|))^{1/4}} \exp(-\sqrt{ \alpha (\beta {+}|b|)})
\]
and
\[
K_{1/2}(\sqrt{\alpha \beta})  =
\frac{(\pi/2)^{1/2}}{( \alpha \beta)^{1/4}} \exp(-\sqrt{ \alpha \beta}),
\]
we obtain the following pdf for ${\EIG}(b|\alpha, \beta, 1/2, 1)$:
\begin{equation} \label{eqn:eig11}
p(b)= \frac{\alpha^{1/2}} {4} \exp(\sqrt{\alpha \beta}) (\beta {+} |b|)^{-1/2}  \exp(-\sqrt{\alpha(\beta {+} |b|)}).
\end{equation}

\paragraph{Example 2}  The second special EP-GIG distribution is based on the setting of $q=1$ and ${\gamma}=3/2$. Since
\[
K_{3/2}(u) = \frac{u+1}{u} K_{1/2}(u) = \frac{u+1}{u} \frac{(\pi/2)^{1/2}}{u^{1/2}} \exp(-u),
\]
we obtain that the pdf of ${\GIG}(\eta|3/2, \beta, \alpha)$ is
\[
p(\eta|\alpha, \beta, 3/2) = \frac{\alpha^{3/2}}{\sqrt{2\pi}} \frac{\exp(\sqrt{\alpha \beta})} { \sqrt{\alpha \beta} +1 }  \eta^{\frac{1}{2}}
\exp(-(\alpha \eta + \beta \eta^{-1})/2)
\]
and that the pdf of ${\EIG}(b|\alpha, \beta, 3/2, 1)$ is
\begin{equation} \label{eqn:eig12}
p(b) =
 \frac{\alpha \exp(\sqrt{\alpha \beta})}{4(\sqrt{\alpha \beta} +1)} \exp(-\sqrt{ \alpha (\beta {+}|b|)}).
\end{equation}

\paragraph{Example 3}  We now consider the case that $q=1$ and ${\gamma}=-1/2$. In this case,
we   have $\EIG(b|\alpha, \beta, {-}1/2, 1)$ which is a mixture of
$L(b|0, \eta)$ with density $\GIG(\eta|{-}1/2,  \beta, \alpha)$.  The density of $\EIG(b|\alpha, \beta, {-}1/2, 1)$
is
\[
p(b)=\frac{\beta^{1/2} \exp(\sqrt{\alpha \beta})} {4 (\beta+|b|)^{3/2}} (1+ \sqrt{\alpha(\beta + |b|)}) \exp(-\sqrt{\alpha(\beta+|b|)}).
\]

\paragraph{Example 4}  The fourth special EP-GIG distribution is ${\EIG}(b|\alpha, \beta, 0, 2)$; that is, we let $q=2$ and $\gamma=0$.
In other words, we consider the mixture of the Gaussian distribution $N(b|0, \eta)$
with the hyperbolic distribution $\GIG(\eta|\beta, \alpha, 0)$. We now have
\begin{equation} \label{eqn:eig22}
p(b) = \frac{1}{2 K_0(\sqrt{\alpha \beta}) \sqrt{\beta+b^2} } \exp(-\sqrt{\alpha(\beta+b^2)}).
\end{equation}

\paragraph{Example 5}   In the fifth special case we set $q=2$ and $\gamma=1$; that is,
we consider the mixture of the Gaussian distribution $N(b|0, \eta)$
with the generalized inverse Gaussian $\GIG(\eta|1, \beta, \alpha)$. The density of the corresponding
EP-GIG distribution ${\EIG}(b|\alpha, \beta, 1, 2)$ is
\begin{equation} \label{eqn:eig21}
p(b) = \frac{1}{2 K_1(\sqrt{\alpha \beta}) \beta^{1/2}} \exp(-\sqrt{\alpha(\beta+b^2)}).
\end{equation}

\paragraph{Example 6}  The final special case is based on the settings $q=2$ and $\gamma=-1$. In this case, we have
\[
p(b) = \int_{0}^{\infty} {N(b|0, \eta)  \GIG(\eta|{-}1, \beta, \alpha) d \eta} = \frac{(\beta/\alpha)^{1/2} } {2 K_1(\sqrt{\alpha \beta})}
\frac{1+\sqrt{\alpha(\beta+b^2)}}{\exp(\sqrt{\alpha(\beta+b^2)})} (\beta + b^2)^{-\frac{3}{2}}.
\]

\paragraph{Example 7}  We are also interested  EP-GIG  with $q=1/2$, i.e. a class of bridge scale mixtures. In this and next examples, 
we present two special cases. First,  we  set $q=1/2$ and ${\gamma}=3/2$. That is,
\[
p(b) = \int_{0}^{\infty} {\EP\big(b|0, \eta, {1}/{2}\big)  \GIG\big(\eta|{3}/{2}, \beta, \alpha\big) d \eta} =
 \frac{\alpha^{\frac{3}{2}}\exp(\sqrt{\alpha \beta})} {2^4 (1+\sqrt{\alpha \beta})}  \frac{\exp(-\sqrt{\alpha(\beta {+} |b|^{1/2})})} 
 {(\beta {+} |b|^{\frac{1}{2}})^{\frac{1}{2}}}.
\]

\paragraph{Example 8}    In this case we set $q=1/2$
and $\gamma=5/2$.  We now have
\[
p(b) = \int_{0}^{\infty} {\EP(b|0, \eta, 1/2)  \GIG(\eta|5/2, \beta, \alpha) d \eta} =  \frac{\alpha^2 \exp(\sqrt{\alpha \beta})}
{2^4(3+ 3\sqrt{\alpha \beta} + \alpha \beta)} \exp\Big(-\sqrt{ \alpha (\beta {+}|b|^{1/2})}\Big).
\]

\section{EP-Jeffreys Priors} \label{app:jeff}

We first consider the definition of EP-Jeffreys prior, which the mixture of  $\EP(b|0, \eta, q)$ with the Jeffreys prior $1/\eta$.
It is easily verified that
\[
p(b) \propto \int{\EP(b|0, \eta, q)  \eta^{-1} d \eta}  = \frac{q}{2} |b|^{-1}
\]
and that $[\eta|b] \sim \IG(\eta|1/q, |b|^q/2)$.
In this case, we obtain
\[
E(\eta^{-1}|b) = \frac{1}{2q} |b|^{-q} .
\]

On the other hand, the EP-Jeffreys prior induces penalty $\log |b|$ for $b$. Moreover, it is immediately  calculated that
\[
\frac{d \log|b|} {|b|^q} \triangleq  \frac{1}{q} |b|^{-q}= 2 E(\eta^{-1}|b).
\]
As we can see,  our discussions here present an alternative derivation for the adaptive lasso~\citep{ZouJasa:2006}. Moreover,
we also obtain the relationship of the adaptive lasso with an EM algorithm.

Using the EP-Jeffreys prior, we in particular define a hierarchical model:
\begin{align*}
[\y|\b, \sigma]  & \sim N(\y|\X \b, \sigma \I_n), \\
[b_j | \eta_j, \sigma] & \stackrel{ind}{\sim}  \EP(b_j|0, \sigma \eta_j, q), \\
[\eta_j] & \stackrel{ind}{\propto}  \eta_j^{-1}, \\
 p(\sigma) & = ``Constant".
\end{align*}
It is easy to obtain that
\[
[\eta_j|b_j, \sigma] \sim \IG\big(\eta_j\big|1/q, \;  \sigma^{-1}|b_j|^q/2 \big).
\]

Given the $t$th estimates $(\b^{(t)}, \sigma^{(t)})$ of $(\b, \sigma)$, the E-step of EM calculates
\begin{align*}
Q(\b, \sigma|\b^{(t)}, \sigma^{(t)}) &\triangleq \log p(\y|\b, \sigma) + \sum_{j=1}^p
\int{\log p[b_j | \eta_j, \sigma] p(\eta_j|b_j^{(t)}, \sigma^{(t)})} d \eta_j \\
& \propto  -\frac{n}{2} \log \sigma {-} \frac{1}{2 \sigma}(\y{-}\X \b)^T (\y {-}\X\b) - \frac{p}{q} \log \sigma \\
& \quad - \frac{1}{2 \sigma} \sum_{j=1}^p
|b_j|^q \int \eta_j^{-1} p(\eta_j|b_j^{(t)}, \sigma^{(t)}) d \eta_j.
\end{align*}
Here we omit some terms that are independent of
parameters $\sigma$ and $\b$.
Indeed, we only need to calculate $E(\eta_j^{-1}|b_j^{(t)}, \sigma^{(t)})$ in the E-step.
That is,
\[
w_j^{(t{+}1)} \triangleq  E(\eta_j^{-1}|b_j^{(t)}, \sigma^{(t)}) = \frac{2 \sigma^{(t)}}
{q |b^{(t)}_j|^q}.
\]

The M-step maximizes $Q(\b, \sigma|\b^{(t)}, \sigma^{(t)})$ with respect to $(\b, \sigma)$.
In particular, it is obtained as follows:
\begin{align*}
\b^{(t{+}1)} & = \argmin_{\b} \; (\y{-}\X \b)^T (\y {-}\X\b) +  \sum_{j=1}^p w_j^{(t{+}1)} |b_j|^q, \\
\sigma^{(t{+}1)} & = \frac{q}{qn {+} 2 p} \Big\{(\y{-}\X \b^{(t{+}1)})^T (\y {-}\X\b^{(t{+}1)}) + \sum_{j=1}^p w_j^{(t{+}1)} 
|b_j^{(t{+}1)}|^q \Big\}.
\end{align*}

\section{The Hierarchy with the Bridge Prior Given in (\ref{eqn:eig13})} \label{app:eig13}

Using the bridge prior in (\ref{eqn:eig13})  yields the following hierarchical model:
\begin{align*}
[\y|\b, \sigma]  & \sim N(\y|\X \b, \sigma \I_n), \\
[b_j | \eta_j, \sigma] & \stackrel{ind}{\sim}  L(b_j|0, \sigma \eta_j), \\
[\eta_j] & \stackrel{ind}{\propto}  G(\eta_j|3/2, \alpha/2), \\
 p(\sigma) & = ``Constant".
\end{align*}
It is easy to obtain that
\[
[\eta_j|b_j, \sigma] \sim \GIG\big(\eta_j\big|1/2, \;  \sigma^{-1}|b_j|, \alpha \big).
\]

Given the $t$th estimates $(\b^{(t)}, \sigma^{(t)})$ of $(\b, \sigma)$, the E-step of EM calculates
\begin{align*}
Q(\b, \sigma|\b^{(t)}, \sigma^{(t)}) &\triangleq \log p(\y|\b, \sigma) + \sum_{j=1}^p
\int{\log p[b_j | \eta_j, \sigma] p(\eta_j|b_j^{(t)}, \sigma^{(t)})} d \eta_j \\
& \propto  -\frac{n}{2} \log \sigma {-} \frac{1}{2 \sigma}(\y{-}\X \b)^T (\y {-}\X\b) - \frac{p}{q} \log \sigma \\
& \quad - \frac{1}{2 \sigma} \sum_{j=1}^p
|b_j|^q \int \eta_j^{-1} p(\eta_j|b_j^{(t)}, \sigma^{(t)}) d \eta_j.
\end{align*}
Here we omit some terms that are independent of
parameters $\sigma$ and $\b$.
Indeed, we only need to calculate $E(\eta_j^{-1}|b_j^{(t)}, \sigma^{(t)})$ in the E-step.
That is,
\[
w_j^{(t{+}1)} \triangleq  E(\eta_j^{-1}|b_j^{(t)}, \sigma^{(t)}) = \sqrt{\frac{\alpha \sigma^{(t)}}
{ |b^{(t)}_j|}}.
\]

The M-step maximizes $Q(\b, \sigma|\b^{(t)}, \sigma^{(t)})$ with respect to $(\b, \sigma)$.
In particular, it is obtained as follows:
\begin{align*}
\b^{(t{+}1)} & = \argmin_{\b} \; (\y{-}\X \b)^T (\y {-}\X\b) +  \sum_{j=1}^p w_j^{(t{+}1)} |b_j|^q, \\
\sigma^{(t{+}1)} & = \frac{q}{qn {+} 2 p} \Big\{(\y{-}\X \b^{(t{+}1)})^T (\y {-}\X\b^{(t{+}1)}) + \sum_{j=1}^p w_j^{(t{+}1)} |b_j^{(t{+}1)}|^q \Big\}.
\end{align*}

\bibliography{ncvs2}

\end{document}